%% file: main.tex
\def\arxivversion{}
\documentclass{article}

\usepackage{iclr2027_conference,times}
\usepackage{amssymb,amsthm,mathtools}
\usepackage{graphicx,booktabs,multirow,capt-of}
\usepackage{tikz}
\usepackage{float}
\usepackage{enumitem}
\usepackage{microtype}
\usepackage{xcolor}
\usepackage{hyperref}
\usepackage{url}
\usepackage[nameinlink,capitalise,noabbrev]{cleveref}

\hypersetup{
  colorlinks=true,
  citecolor=blue!55!black,
  linkcolor=blue!55!black,
  urlcolor=blue!55!black
}
\newtheorem{theorem}{Theorem}
\newtheorem{lemma}{Lemma}
\newtheorem{proposition}{Proposition}
\newtheorem{corollary}{Corollary}

\newcommand{\R}{\mathbb{R}}
\newcommand{\E}{\mathbb{E}}
\newcommand{\cA}{\mathcal{A}}
\newcommand{\cD}{\mathcal{D}}
\newcommand{\cG}{\mathcal{G}}
\newcommand{\cK}{\mathcal{K}}
\newcommand{\cL}{\mathcal{L}}
\newcommand{\cM}{\mathcal{M}}
\newcommand{\cP}{\mathcal{P}}
\newcommand{\cT}{\mathcal{T}}
\newcommand{\cW}{\mathcal{W}}
\newcommand{\Proj}{\operatorname{Proj}}
\newcommand{\dist}{\operatorname{dist}}
\newcommand{\argmaxop}{\operatorname*{arg\,max}}

\newlength{\paperfigurewidth}
\title{Optimal Training-Time Scaling\\in Gradual Adaptation}
\hypersetup{pdftitle={Optimal Training-Time Scaling in Gradual Adaptation}}
\ifdefined\arxivversion
  
  \author{Zonghuan Xu\\Fudan University\And Krishna Harish\\Houston Community College}
  \iclrfinalcopy
  \hypersetup{pdfauthor={Zonghuan Xu, Krishna Harish}}
\else
  
  \author{Anonymous authors}
  \hypersetup{pdfauthor={}}
\fi
\makeatletter
\AtBeginDocument{\typeout{FORMAT-AUDIT: textwidth=\the\textwidth; textheight=\the\textheight; fontsize=\f@size; baseline=\the\baselineskip; parskip=\the\parskip; oddsidemargin=\the\oddsidemargin; figurewidth=\the\paperfigurewidth}}
\makeatother
\begin{document}
\maketitle
\ifdefined\arxivversion
\lhead{Preprint}
\fi

\begin{abstract}
In gradual adaptation, how should the time spent training on each task change as tasks are sampled more frequently along the same transition?
We study this question for overparameterized linear regression tasks that change smoothly and share a zero-loss solution.
With $N$ tasks and training time $s_N$ on each, the final learning progress converges to a continuum curve when $Ns_N\to\tau$.
The limiting progress is $\Theta(\tau)$ for small $\tau$ and $\Theta(\tau^{-1})$ for large $\tau$, so both very short and very long training produce little progress.
There exists a constant $\tau_\star>0$, independent of $N$, such that training for $\tau_\star/N$ on each task achieves learning progress within $o(1/N)$ of the maximum over all per-task training times.
Furthermore, for smooth nonlinear networks and any smooth evaluation loss, the $1/N$ law holds in at least one case per pair whenever the transition is nonsingular and finite-time training can strictly lower the loss in both, even without a shared solution.
On Rotated MNIST and Rotated Fashion-MNIST, the optimal total training time stays nearly constant as sampling frequency increases across the models, losses, and evaluation metrics tested at small learning rates, allowing a budget tuned at one frequency to predict the optimal training time per task at another.
Keeping the total training budget fixed is therefore a natural first choice as the sampling frequency increases.
\end{abstract}

\section{Introduction}

Many learning systems must process sequences of slowly changing training tasks.  In gradual domain
adaptation, intermediate distributions guide a learner from a source domain to a target domain
\citep{kumar2020gradual,chen2021indexed,wang2022gradual}.  In curriculum learning, related tasks are
ordered by difficulty or structure
\citep{bengio2009curriculum,kumar2010selfpaced,graves2017curriculum,weinshall2018curriculum,
hacohen2019power,klink2022transport,huang2022gradient}.  In continual test-time adaptation, a model
is repeatedly updated as its deployment data evolve
\citep{sun2020testtime,wang2021tent,wang2022cotta,niu2022eata,zhang2022memo,
niu2023sar,dobler2023rmt,yuan2023dynamic}.  We use \emph{gradual adaptation} for this setting of
sequential learning over nearby tasks, which also falls within the broader scope of continual
learning \citep{parisi2019continual,delange2021continual,vandeven2022three}. See Appendix~\ref{app:related-work} for a detailed discussion of related work.

One motivation comes from embodied agents that learn from experience as it unfolds. When conditions
change gradually, an agent that collects observations at shorter intervals can learn from intermediate
stages that a coarser sampling schedule would skip. Yet successive batches then come from more similar
conditions. Should the agent spend as long training on each batch when less has changed since the last one?

In an interacting system, what the agent learns can also change what it observes next. To study the
relationship between sampling frequency and training time in isolation, we fix the underlying
transition and vary how densely tasks are sampled along it. Suppose the
sequence contains $N$ tasks and the learner trains for time $s_N$ on each task.  The task count $N$
controls the spacing between adjacent tasks, while $s_N$ controls the amount of optimization at each
task.  We ask: \textbf{As the number of tasks $N$ increases, how long should each task be trained to
maximize the learning progress after the entire sequence?}

We study this problem in overparameterized linear regression with population squared losses and a
shared zero-loss solution. We sample $N$ tasks at equally spaced positions $t_j=j/N$ along a fixed
smooth transition indexed by $t\in[0,1]$, starting from a model that fits the first task and running
gradient flow for time $s_N$ on each subsequent task. Learning progress is the reduction in squared
distance to the common solution set over the complete sequence. We denote this progress by
$\cP_{N,s_N}$ and choose the per-task training time $s_N^\star$ to maximize it.

Our first theorem shows that, as the task count grows with $Ns_N\to\tau<\infty$, learning progress
converges to a limiting curve $\cP_\tau$ indexed by the total training time $\tau$. This curve
provides a common basis for comparing different task counts and per-task training times.

Our second theorem characterizes both ends of this curve. Whenever the initial model is not
already in the common solution set,
\[
  \cP_\tau=\Theta(\tau)\quad(\tau\downarrow0),
  \qquad
  \cP_\tau=\Theta(\tau^{-1})\quad(\tau\to\infty).
\]
With little training, the model barely moves. With extensive training, it closely follows the
current task's zero-loss subspace as that subspace rotates, mainly changing the direction of the
remaining error rather than reducing its length. Thus, both insufficient training and overtraining
yield vanishing progress.

Only total training times of order one can retain positive limiting progress, and we prove that
the optimal per-task time satisfies $s_N^\star=\Theta(N^{-1})$. We further show that there exists a
constant $\tau_\star>0$, independent of $N$, such that training for $\tau_\star/N$ on each task achieves
learning progress within $o(N^{-1})$ of the maximum over all per-task training times.

So far, we have established the scaling law for linear regression, with progress measured by the reduction in squared distance to the common solution set. To understand its scope, we first consider how the choice of evaluation criterion affects the optimal training time. We construct a linear regression path on which the total training time that maximizes learning progress stays bounded, while the total training time that minimizes average population squared loss diverges as $N$ grows.

This counterexample shows that the $1/N$ law cannot hold universally across evaluation criteria. Our third theorem nevertheless gives a guarantee for any smooth evaluation loss, even for smooth nonlinear networks whose tasks have no shared solution. If the transition is nonsingular and finite-time training can strictly lower the evaluation loss in both members of a pair, then at least one of the two adaptation processes has an optimal per-task training time of order $1/N$.

This result suggests that the scaling may hold more broadly, motivating us to test it across different configurations. On Rotated MNIST and Rotated Fashion-MNIST, the optimal total training time remained nearly constant as sampling frequency increased, allowing a budget tuned at 64 tasks to predict the optimal per-task training time at 512 tasks to within 8\% in all 30 small-learning-rate settings.

\paragraph{Contributions.}
\begin{itemize}
  \item We formulate a training-time allocation problem for gradual adaptation by asking how the optimal time spent on each task changes as tasks are sampled more frequently along the same transition.
  \item For linear regression tasks with a shared solution, we prove that there exists a constant $\tau_\star>0$, independent of $N$, such that training for $\tau_\star/N$ on each task achieves learning progress within $o(1/N)$ of the maximum attainable by tuning the per-task training time separately for each $N$.
  \item A counterexample shows that switching the evaluation criterion from distance to the common solution set to prediction loss can break the $1/N$ scaling even in linear regression, yet for smooth nonlinear networks and any smooth evaluation loss, we prove that at least one process in each pair obeys the scaling whenever the transition is nonsingular and finite-time training can strictly lower the loss in both processes.
  \item On Rotated MNIST and Rotated Fashion-MNIST, the optimal total training time stays nearly constant as sampling frequency increases across the models, losses, and evaluation metrics tested at small learning rates, allowing a budget tuned at one frequency to predict the optimal training time per task at another.
\end{itemize}

\section{Problem Setup}
\label{sec:setup}

This section sets up our analysis of overparameterized linear regression, where we measure learning
progress by the reduction in squared distance to the common solution set. \Cref{sec:paired-budgets}
introduces the assumptions for nonlinear models and more general evaluation losses.

\subsection{A smooth family of regression tasks}

Let $t\in[0,1]$ index a gradually changing task.  On task $t$, a feature--label pair
$(x,y)\in\R^d\times\R$ follows a distribution $\cD_t$.  We assume that there is a parameter
$w^\star\in\R^d$ such that $y=x^\top w^\star$ almost surely under every $\cD_t$.  The population
squared loss is therefore
\begin{equation}
  \begin{aligned}
    \cL_t(w)
    &:=\frac12\E_{\cD_t}\!\left[(x^\top w-y)^2\right] \\
    &=\frac12(w-w^\star)^\top H(t)(w-w^\star),
    \qquad H(t):=\E_{\cD_t}[xx^\top].
  \end{aligned}
  \label{eq:quadratic-loss}
\end{equation}
Here $H(t)$ is both the population Hessian and the feature second moment.  For every $z\in\R^d$,
\begin{equation}
  z^\top H(t)z
  =\E_{\cD_t}\!\left[(x^\top z)^2\right]\ge0,
  \label{eq:psd-explanation}
\end{equation}
so $H(t)$ is positive semidefinite.  Its kernel consists of parameter perturbations whose linear
predictions remain unchanged on task $t$:
$V(t):=\ker H(t) =\{z:x^\top z=0\ \text{almost surely under }\cD_t\}.$
A parameter change along $z\in V(t)$ leaves the predictions on task $t$ unchanged.  For a finite
regression dataset with design matrix $X(t)\in\R^{n_t\times d}$ and labels
$y(t)=X(t)w^\star$, the same representation holds with
$H(t)=X(t)^\top X(t)/n_t$.  The overparameterized regime corresponds to
$\operatorname{rank}H(t)<d$, so each task leaves some parameter directions unconstrained.

We assume that $H:[0,1]\to\R^{d\times d}$ is a $C^3$ path of symmetric positive-semidefinite
matrices with constant rank.  Its positive eigenvalues satisfy the uniform bounds
\begin{equation}
  0<\mu\le\lambda\le L<\infty.
  \label{eq:spectral-gap}
\end{equation}
Under these assumptions, the orthogonal projector onto $V(t)$ is $C^3$ in $t$
(\cref{app:nullspace-smoothness}). Let
$P(t)=\Proj_{V(t)}, \quad Q(t)=I-P(t), \quad \cK=\bigcap_{t\in[0,1]}V(t).$
Thus $P(t)$ projects onto the parameter directions unconstrained by task $t$, while $Q(t)$ projects
onto their orthogonal complement.  The minimizer set of task $t$ is
$\cM(t)=w^\star+V(t)$, and the solution set shared by the full path is $w^\star+\cK$.

\subsection{Sequential training and learning progress}

Select tasks on the uniform grid $t_j=j/N$, $j=0,\ldots,N$, and write $v=w-w^\star$.  On task
$t_j$, gradient flow for \cref{eq:quadratic-loss} satisfies
$\frac{d}{du}v(u)=-H(t_j)v(u).$
Training for time $s\ge0$ therefore applies the exact update
\begin{equation}
  v_j=e^{-sH(t_j)}v_{j-1}.
  \label{eq:task-update}
\end{equation}
Here $s$ denotes the gradient-flow time spent on each task.  A run of $k$ gradient-descent steps with
learning rate $\eta$ approximates \cref{eq:task-update} with $s\approx\eta k$ in the small-step
regime.  For task-wise fitting to convergence, we define
$e^{-\infty H(t)}:=P(t)$.

On the grid above, suppose the learner processes $t_1,\ldots,t_N$ in order and runs gradient flow for
the same time $s$ on every task.  Repeated application of \cref{eq:task-update} gives
\begin{equation}
  \Phi_{N,s}
  :=e^{-sH(t_N)}e^{-sH(t_{N-1})}\cdots e^{-sH(t_1)},
  \qquad v_N=\Phi_{N,s}v_0.
  \label{eq:finite-product}
\end{equation}
We begin immediately after the initial task has been fitted, so $v_0\in V(0)$.  This convention
focuses the analysis on the learning progress contributed by the $N$ subsequent tasks.

Decompose the initial error orthogonally as
$v_0=v_0^\parallel+v_0^\perp, \quad v_0^\parallel\in\cK, \quad v_0^\perp\in\cK^\perp.$
Since $\cK\subseteq\ker H(t)$ for every $t$, each update fixes $v_0^\parallel$ and preserves
$\cK^\perp$.  Hence the distances to the common solution set before and after training are
$\dist(w^\star+v_0,w^\star+\cK)=\|v_0^\perp\|, \quad \dist(w^\star+\Phi_{N,s}v_0,w^\star+\cK) =\|\Phi_{N,s}v_0^\perp\|.$
Learning progress therefore depends only on $v_0^\perp$.  To simplify notation, we now rename
$v_0^\perp$ as $v_0$.  Together with the assumption that the initial task has been fitted,
this gives $v_0\in V(0)\cap\cK^\perp$.

In this setting, learning progress after $N$ tasks is the reduction in squared distance to the
common solution set,
\begin{equation}
  \cP_{N,s}(v_0)=\|v_0\|^2-\|\Phi_{N,s}v_0\|^2\ge0,
  \qquad v_0\in V(0)\cap\cK^\perp.
  \label{eq:progress}
\end{equation}

This distance has a uniform loss interpretation over the task path.  Define the worst population
loss of a model $w$ along the path by
$\mathcal{R}_{\mathrm{path}}(w):=\sup_{t\in[0,1]}\cL_t(w).$

\begin{proposition}[Loss interpretation of the learning-progress measure]
\label{prop:pathwise-risk-equivalence}
Suppose $\cK^\perp\ne\{0\}$, and let
$\bar H:=\int_0^1 H(t)\,dt, \quad c_H:=\frac12\lambda_{\min}\!\left(\bar H\big|_{\cK^\perp}\right).$
Then $c_H>0$ and, for every $w\in\R^d$,
\begin{equation}
  c_H\,\dist(w,w^\star+\cK)^2
  \le \mathcal{R}_{\mathrm{path}}(w)
  \le \frac{L}{2}\,\dist(w,w^\star+\cK)^2.
  \label{eq:pathwise-risk-equivalence}
\end{equation}
Consequently, $\dist(w_m,w^\star+\cK)\to0$ for a sequence $(w_m)$ if and only if
$\sup_{t\in[0,1]}\cL_t(w_m)\to0$.
\end{proposition}

Thus squared distance to the common solution set is equivalent, up to path-dependent constants, to
the worst population loss along the task path.  \Cref{eq:progress} records the reduction of this
state measure along the task sequence.  The proof is in \cref{app:pathwise-risk-equivalence}.

\begin{proposition}[Exact decomposition of learning progress]
\label{prop:exact-progress-decomposition}
For every $N$ and $s\in[0,\infty]$,
\begin{equation}
  \cP_{N,s}(v_0)
  =\sum_{j=1}^N
    \left\langle v_{j-1},
    \bigl(I-e^{-2sH(t_j)}\bigr)v_{j-1}\right\rangle.
  \label{eq:exact-progress-decomposition}
\end{equation}
\end{proposition}
\begin{proof}
Since $H(t_j)$ is symmetric, \cref{eq:task-update} gives
\[
  \|v_{j-1}\|^2-\|v_j\|^2
  =\left\langle v_{j-1},
    \bigl(I-e^{-2sH(t_j)}\bigr)v_{j-1}\right\rangle.
\]
Summing over $j$ telescopes to \cref{eq:progress}.
\end{proof}

Each summand is exactly the reduction in squared error produced on one task.
\Cref{eq:exact-progress-decomposition} exposes two coupled roles of the per-task training time.  A
larger $s$ removes more error on the current task, while also changing the state inherited by every
later task.  Not fitting the current task completely leaves a residual on which subsequent, slightly
different tasks can act.  Across a dense task sequence, these later reductions can accumulate and
produce more progress over the full path than exact fitting.  \Cref{sec:results} quantifies this
effect.

For each task count $N$, an optimal per-task training time for this criterion maximizes
$\cP_{N,s}(v_0)$ over $s$,
\begin{equation}
  s_N^\star(v_0)\in\argmaxop_{s\in[0,\infty]}\cP_{N,s}(v_0).
  \label{eq:finite-optimizer}
\end{equation}
Because $\Phi_{N,s}$ has a limit as $s\to\infty$, the maximum in
\cref{eq:finite-optimizer} is attained.

\section{Training-Time Scaling in Linear Regression}
\label{sec:results}

We first pass from the discrete task sequence to a continuum learning curve indexed by the effective
training time $\tau=Ns_N$.  We then characterize its small- and large-budget limits and use them to
identify the optimal scaling of the per-task training time.

\subsection{The limiting learning curve}

For an effective training time $\tau\ge0$, let $x_\tau:[0,1]\to\R^d$ solve
\begin{equation}
  \dot x_\tau(t)=-\tau H(t)x_\tau(t),
  \qquad x_\tau(0)=v_0,
  \label{eq:main-continuum-flow}
\end{equation}
and define
\begin{equation}
  \cP_\tau(v_0):=\|v_0\|^2-\|x_\tau(1)\|^2.
  \label{eq:main-continuum-progress}
\end{equation}

\begin{theorem}[Limit of dense task sequences]
\label{thm:continuum-main}
Under the assumptions of \cref{sec:setup}, if $Ns_N\to\tau\in[0,\infty)$, then
\begin{equation}
  \Phi_{N,s_N}v_0\longrightarrow x_\tau(1),
  \qquad
  \cP_{N,s_N}(v_0)\longrightarrow\cP_\tau(v_0).
  \label{eq:main-continuum-limit}
\end{equation}
For every $T<\infty$, the finite-sequence approximation satisfies
\begin{align}
  \sup_{0\le\tau\le T}
  \left\|\Phi_{N,\tau/N}v_0-x_\tau(1)\right\|
  &=O_T(N^{-1}),
  \label{eq:main-continuum-rate}\\
  \sup_{0\le\tau\le T}
  \left|\cP_{N,\tau/N}(v_0)-\cP_\tau(v_0)\right|
  &=O_T(N^{-1}).
  \label{eq:main-progress-rate}
\end{align}
Moreover,
\begin{equation}
  \cP_\tau(v_0)
  =2\tau\int_0^1x_\tau(t)^\top H(t)x_\tau(t)\,dt.
  \label{eq:main-continuum-energy}
\end{equation}
\end{theorem}

\paragraph{Proof sketch.}
Put $b_N=Ns_N$.  Since $b_N\to\tau$, each update satisfies
$e^{-s_NH(t_j)} =I-\frac{b_N}{N}H(t_j)+O(N^{-2})$
uniformly along the path.  The one-step error relative to \cref{eq:main-continuum-flow} is
$O(N^{-2})+O(|b_N-\tau|/N)$.  Because every update is a contraction, these errors are not amplified
by the remaining product.  After $N$ steps the endpoint error is therefore
$O(N^{-1})+O(|b_N-\tau|)$, which tends to zero.  When $b_N=\tau$ on a compact interval, this also
gives \cref{eq:main-continuum-rate}.  The progress rate follows because all trajectories are
uniformly bounded.  Differentiating $\|x_\tau(t)\|^2$ and integrating from $0$ to $1$ gives
\cref{eq:main-continuum-energy}.  Complete estimates are in \cref{app:technical}.

\subsection{Optimal per-task training time}

The condition $v_0\in V(0)\cap\cK^\perp$ means that the initial point fits the source task and that
components invisible to every task have been removed.  Hence a nonzero $v_0$ is detected by at
least one later task.

\begin{theorem}[Optimal per-task training time]
\label{thm:optimal-main}
Let $v_0\in V(0)\cap\cK^\perp$ be nonzero.  Then:
\begin{enumerate}[label=(\roman*)]
  \item The limiting progress is strictly positive for every finite $\tau>0$ and satisfies
  \begin{equation}
    \cP_\tau(v_0)=\Theta(\tau)\quad(\tau\downarrow0),
    \qquad
    \cP_\tau(v_0)=\Theta(\tau^{-1})\quad(\tau\to\infty).
    \label{eq:main-progress-tails}
  \end{equation}
  \item For every sequence of per-task training times $s_N$,
  \begin{equation}
  \begin{array}{rcl}
    Ns_N\to0
    &\Longrightarrow&\cP_{N,s_N}(v_0)\to0,\\[1mm]
    Ns_N\to\tau\in(0,\infty)
    &\Longrightarrow&\cP_{N,s_N}(v_0)\to\cP_\tau(v_0)>0,\\[1mm]
    Ns_N\to\infty
    &\Longrightarrow&\cP_{N,s_N}(v_0)\to0.
  \end{array}
  \label{eq:main-three-regimes}
  \end{equation}
  \item The limiting curve attains its maximum at one or more points in $(0,\infty)$.  For all
  sufficiently large $N$, every optimizer in \cref{eq:finite-optimizer} is finite and satisfies
  $s_N^\star(v_0)=\Theta(N^{-1})$, and every accumulation point of $Ns_N^\star(v_0)$ maximizes
  $\cP_\tau(v_0)$.  In particular, if the limiting curve has a unique maximizer
  $\tau^\star(v_0)$, then
  \begin{equation}
    s_N^\star(v_0)
    =\frac{\tau^\star(v_0)}{N}+o(N^{-1}).
    \label{eq:main-unique-optimum}
  \end{equation}
  \item There exists a maximizer $\tau_\star>0$ of the limiting progress curve such that
  \[
    0\le \cP_{N,s_N^\star}(v_0)-\cP_{N,\tau_\star/N}(v_0)=o(N^{-1}).
  \]
  Thus one fixed total budget suffices even when the limiting curve has multiple maximizers.
\end{enumerate}
\end{theorem}

In particular, fitting every task to convergence belongs to the third regime in
\cref{eq:main-three-regimes}: its progress tends to zero as $N\to\infty$.  By contrast,
$s_N=\tau/N$ with $0<\tau<\infty$ retains strictly positive limiting progress.

\paragraph{Proof idea.}
The distinction is easiest to see by comparing complete trajectories rather than a single update.
See \cref{fig:partial-fitting-mechanism}.  Put $h=N^{-1}$.  After exact fitting,
$v_j\in V(t_j)$.  Because adjacent zero-loss spaces differ by only $O(h)$, the next task exposes
only an $O(h)$ normal component.  Projecting this component away turns the vector by $O(h)$, but
changes its squared length by only $O(h^2)$.  Across $N=h^{-1}$ tasks, the vector therefore mainly
follows the moving zero-loss space along an almost constant-radius path, with only $O(N^{-1})$
total progress.

\begin{figure}[H]
  \centering
  \resizebox{\paperfigurewidth}{!}{%
  \begin{tikzpicture}[
      x=1cm,y=1cm,
      task/.style={draw={rgb,255:red,47;green,89;blue,128},line width=.66pt},
      exact/.style={draw={rgb,255:red,172;green,86;blue,26},line width=.72pt,->},
      exactpath/.style={draw={rgb,255:red,172;green,86;blue,26},line width=.58pt,densely dashed},
      finite/.style={draw={rgb,255:red,22;green,117;blue,81},line width=.72pt,->},
      finitepath/.style={draw={rgb,255:red,22;green,117;blue,81},line width=.62pt},
      lab/.style={font=\scriptsize,align=center},
      every node/.style={inner sep=1.2pt}]
    \draw[exactpath] (2.70,2.15)--(3.12,2.15);
    \node[lab,anchor=west] at (3.18,2.15) {exact fitting};
    \draw[finitepath] (5.00,2.15)--(5.42,2.15);
    \node[lab,anchor=west] at (5.48,2.15) {$s_N=\tau^\star/N$};
    \draw[task] (8.00,2.15)--(8.42,2.15);
    \node[lab,anchor=west] at (8.48,2.15) {$V(t)$};

    \begin{scope}[shift={(0,0)}]
      \begin{scope}[shift={(1.85,0)},scale=1.70]
        \draw[task] (0,0)--(-0.575,0.996);
        \draw[exactpath] plot[smooth] coordinates
          {(0.0000,1.0000) (-0.1305,0.9914) (-0.2588,0.9659)
           (-0.3827,0.9239) (-0.5000,0.8660)};
        \draw[finitepath] plot[smooth] coordinates
          {(0.0000,1.0000) (-0.0112,0.9990) (-0.0417,0.9928)
           (-0.0862,0.9775) (-0.1396,0.9510)};
        \draw[exact] (0,0)--(-0.5000,0.8660);
        \draw[finite] (0,0)--(-0.1396,0.9510);
        \fill[black] (0,0) circle (1.1pt);
      \end{scope}
      \node[lab] at (0.88,-0.43) {(a) $t=1/3$};
    \end{scope}

    \begin{scope}[shift={(4.55,0)}]
      \begin{scope}[shift={(1.85,0)},scale=1.70]
        \draw[task] (0,0)--(-0.996,0.575);
        \draw[exactpath] plot[smooth] coordinates
          {(0.0000,1.0000) (-0.1305,0.9914) (-0.2588,0.9659)
           (-0.3827,0.9239) (-0.5000,0.8660) (-0.6088,0.7934)
           (-0.7071,0.7071) (-0.7934,0.6088) (-0.8660,0.5000)};
        \draw[finitepath] plot[smooth] coordinates
          {(0.0000,1.0000) (-0.0112,0.9990) (-0.0417,0.9928)
           (-0.0862,0.9775) (-0.1396,0.9510) (-0.1971,0.9125)
           (-0.2543,0.8623) (-0.3074,0.8016) (-0.3536,0.7325)};
        \draw[exact] (0,0)--(-0.8660,0.5000);
        \draw[finite] (0,0)--(-0.3536,0.7325);
        \fill[black] (0,0) circle (1.1pt);
      \end{scope}
      \node[lab] at (0.88,-0.43) {(b) $t=2/3$};
    \end{scope}

    \begin{scope}[shift={(9.10,0)}]
      \begin{scope}[shift={(1.85,0)},scale=1.70]
        \draw[task] (0,0)--(-1.15,0);
        \draw[exactpath] plot[smooth] coordinates
          {(0.0000,1.0000) (-0.1305,0.9914) (-0.2588,0.9659)
           (-0.3827,0.9239) (-0.5000,0.8660) (-0.6088,0.7934)
           (-0.7071,0.7071) (-0.7934,0.6088) (-0.8660,0.5000)
           (-0.9239,0.3827) (-0.9659,0.2588) (-0.9914,0.1305)
           (-1.0000,0.0000)};
        \draw[finitepath] plot[smooth] coordinates
          {(0.0000,1.0000) (-0.0112,0.9990) (-0.0417,0.9928)
           (-0.0862,0.9775) (-0.1396,0.9510) (-0.1971,0.9125)
           (-0.2543,0.8623) (-0.3074,0.8016) (-0.3536,0.7325)
           (-0.3907,0.6573) (-0.4175,0.5784) (-0.4334,0.4986)
           (-0.4385,0.4202)};
        \draw[exact] (0,0)--(-1.0000,0.0000);
        \draw[finite] (0,0)--(-0.4385,0.4202);
        \fill[black] (0,0) circle (1.1pt);
      \end{scope}
      \node[lab] at (0.88,-0.43) {(c) $t=1$};
    \end{scope}
  \end{tikzpicture}}
  \caption{Limiting trajectories for the rotating-task example in \cref{app:rotating-example}.
  All panels use the same scale.  The black dot is the common solution and the blue ray is the
  current zero-loss space.  Exact fitting follows this space along the outer orange arc without
  approaching the dot.  At the optimal scale $s_N=\tau^\star/N$, finite training lags behind the
  current zero-loss space but moves steadily inward.  By the final task it is visibly closer to the
  common solution.}
  \label{fig:partial-fitting-mechanism}
\end{figure}

With $s_N=\tau h$, no task removes its visible component completely:
$e^{-s_NH(t_j)}v_{j-1} =v_{j-1}-\tau hH(t_j)v_{j-1}+O(h^2).$
The remaining component need not shrink to $O(h)$ before the next task arrives.  Consequently the
$O(h)$ decreases in squared length can accumulate over $N=h^{-1}$ tasks, giving the nonzero limit
in \cref{eq:main-continuum-energy}.  In the figure this is the difference between merely
following the outer arc and moving inward.

The two failures occur on opposite sides of this balance.  If $Ns_N\to0$, the total amount of
training vanishes.  If $Ns_N\to\infty$, the model is driven increasingly close to each current
zero-loss space.  The appendix proves the uniform estimate
\[
  \cP_{N,s_N}(v_0)
  =O\!\left(\frac{1}{N(1-e^{-\mu s_N})}\right).
\]
Its right-hand side tends to zero whenever $Ns_N\to\infty$.
Fix any finite $\bar\tau>0$.  The choice $s_N=\bar\tau/N$ gives strictly positive limiting progress.
The two tail bounds rule out $Ns_N^\star\to0$ and $Ns_N^\star\to\infty$, so every optimizer satisfies
$Ns_N^\star=\Theta(1)$.  Uniform convergence on this bounded range then implies that every
accumulation point of $Ns_N^\star$ maximizes $\cP_\tau(v_0)$.  Complete expansions and uniform
remainder estimates are given in
\cref{app:tail-details,app:finite-time-proof}.  For part (iv), we expand discrete progress to first order in $1/N$
on this compact interval and select a limiting maximizer with the largest first-order correction
(\cref{app:budget-correction}).

\paragraph{Finite training versus exact fitting.}
Even outside the optimal $1/N$ regime, finite training can outperform exact fitting. For every
fixed finite per-task time $s>0$, both have progress of order $N^{-1}$, but finite training has a
strictly larger leading coefficient. It leaves residual error that later tasks can still reduce.
The full expansions and coefficient comparison are given in \cref{app:finite-time-proof}.

\paragraph{Robustness to monotone distance transformations.}
Replacing squared distance by any strictly increasing function of it preserves the ranking of
training times and hence their optimizers. This need not hold for anisotropic distances; the
identity establishing the invariance and an explicit reversal example are given in
\cref{app:anisotropic-counterexample}.

\section{Scope of the Training-Time Scaling Law}
\label{sec:scope}

\Cref{sec:results} establishes the $1/N$ law for linear regression when progress is measured by
 distance to the common solution set. We first show that changing the evaluation criterion alone
can change the optimal training scale. We then give a guarantee for smooth nonlinear models and
any smooth evaluation loss. Under the conditions below, at least one run in each pair obeys the
$1/N$ law. Finally, we use experiments to test how consistently the law holds across models,
training losses, and evaluation criteria.

\subsection{A counterexample for average prediction loss}
\label{sec:risk-counterexample}

Changing the evaluation criterion can change the optimal training scale even in linear regression.
We construct a task path satisfying the assumptions of \cref{thm:optimal-main} on which the total
training time that minimizes distance to the common solution stays of order one, whereas every
total training time that minimizes average population squared loss diverges as the task count grows.
The construction has four steps, illustrated in \cref{fig:risk-counterexample}.

\begin{figure}[t]
\centering
\includegraphics[width=\paperfigurewidth]{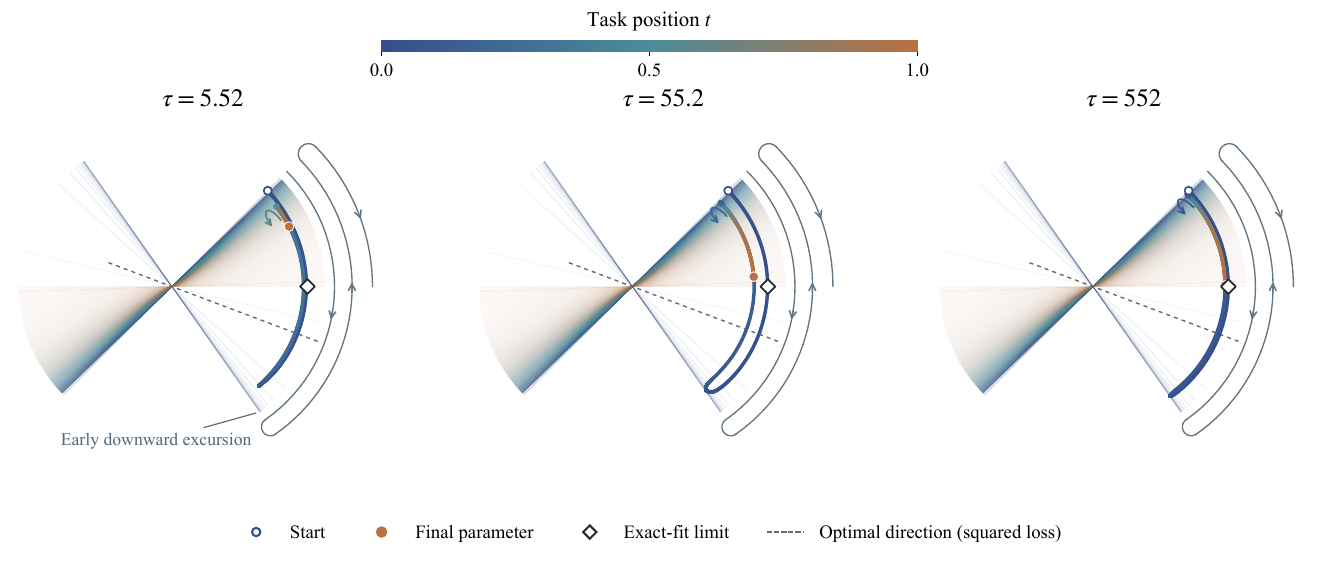}
\caption{The counterexample path with $a=0.6$ and $\varepsilon=3\times10^{-4}$.
Faint lines are task solution sets sampled uniformly along the path, and colored curves are
parameter trajectories. Both use the same time colors. Outer arrows indicate the task line's
downward excursion, return upward, and final turn to horizontal. Small curved arrows mark
parameter reversals. The dashed line gives the direction minimizing average squared loss at
fixed length. Finite-training endpoints lie above the horizontal exact-fit limit, which has lower
average prediction loss. A common nonlinear angular magnification and radial magnification make
the small directional and length differences visible in all panels.}
\label{fig:risk-counterexample}
\end{figure}

\paragraph{Step 1. Distinguish the effects of length and direction.}
Place the common solution at the origin and write the final parameter as $v=\rho d$, where
$\|d\|=1$. Squared distance to the common solution and average population squared loss are
\[
  \dist(v,\{0\})^2=\rho^2,
  \qquad
  R(v)=\frac12\rho^2d^\top\bar H d,
\]
where $\bar H=\int_0^1H(t)\,dt$ averages the task Hessians along the same path. Shortening the
error improves the distance criterion, while prediction loss also depends on the directional
weight $d^\top\bar H d$. We seek a path on which finite training gains from shortening the error
but incurs a larger loss through its final direction.

\paragraph{Step 2. Use rotating zero-loss lines to create tracking lag.}
Let each task have a zero-loss line through the origin, and initialize at a unit vector on the
first line. First consider rotating the line gradually from slightly above horizontal to horizontal.
As tasks become dense, exact fitting follows the moving line to a horizontal endpoint while
preserving almost all of the initial length. Finite training can shorten the parameter, but
tracking lag leaves its endpoint slightly above horizontal. This creates the desired difference
in length and direction. However, if all task lines lie above horizontal, this directional offset
can also help fit earlier tasks. We therefore modify the path so that the upward offset hurts
average prediction loss.

\paragraph{Step 3. Add an early downward excursion.}
Let the zero-loss line briefly turn below horizontal, return upward, and then gradually turn to
horizontal. The early excursion enters the average prediction loss and changes $\bar H$, making
the optimal direction at a fixed length slightly downward. By placing the excursion sufficiently
early and controlling its amplitude, we can keep the terminal parameter above horizontal for
every finite total training time.

Specifically, choose a nonnegative smooth function $\beta$ supported in $(0,1/10)$ with integral
one. Set
\[
  \theta(t)=\varepsilon\bigl(1-t-a\beta(t)\bigr),
  \qquad
  u(t)=\begin{pmatrix}\cos\theta(t)\\\sin\theta(t)\end{pmatrix},
  \qquad
  n(t)=\begin{pmatrix}-\sin\theta(t)\\\cos\theta(t)\end{pmatrix},
\]
and define
$L_t(v)=\frac12\bigl(n(t)^\top v\bigr)^2, \quad v_0=u(0).$
The zero-loss set of each task is $\operatorname{span}(u(t))$, the only common solution is the
origin, and the nonzero eigenvalue of every task Hessian is one.

The coefficient $a$ controls both effects of the excursion. Since
$\int_0^1\theta(t)\,dt=\varepsilon\left(\frac12-a\right),$
choosing $a>1/2$ puts the optimal prediction direction slightly below horizontal when
$\varepsilon$ is small. On the other hand, \cref{app:risk-order-example} shows that $a<0.4e$
suffices to control the excursion's effect on the terminal parameter, keeping it above horizontal
for every finite budget when $\varepsilon$ is sufficiently small. We take $a=0.6$, which meets
both requirements.

\paragraph{Step 4. Reduce the angles so that the directional cost dominates.}
Choose a sufficiently small fixed $\varepsilon>0$, bringing all zero-loss lines close to horizontal.
This weakens the benefit of length reduction in two ways. First, training at small angles shortens
the parameter only slightly. Second, the horizontal direction already fits all tasks nearly
perfectly, so further shortening saves little prediction loss. The benefit is thus a product of
two small quantities. Both the reduction in squared length and the loss weight in the horizontal
direction are of order at most $\varepsilon^2$, giving a benefit of order at most
$\varepsilon^4$. The directional cost is of order $\varepsilon^2$. Reducing the angles therefore
makes the length benefit decrease faster, allowing the directional cost to dominate.

\Cref{app:risk-order-example} makes this comparison uniform over all finite total training times:
\[
  R(v_\tau)>R(e_1)\qquad\text{for every finite }\tau\ge0,
\]
where $v_\tau$ is the continuum endpoint at total time $\tau$ and $e_1$ is the exact-fitting limit.
Thus every finite total time has higher risk than the exact-fitting limit. \Cref{thm:continuum-main}
then rules out bounded risk-optimal total times, whereas \cref{thm:optimal-main} keeps the
distance-optimal total time of order one. The full argument is given in \cref{app:risk-order-example}.

\subsection{A scaling guarantee for paired runs}
\label{sec:paired-budgets}

Although the counterexample rules out an unconditional guarantee for every evaluation criterion,
pairing training processes gives a broader result: for smooth nonlinear models and any smooth
evaluation loss, \textbf{at least one process in each pair obeys the $1/N$ law} under the conditions below.

We retain the equally spaced tasks and equal per-task gradient-flow times of \cref{sec:setup}.
Each task admits an exact fit, but the tasks need not share a solution. The model may be nonlinear
in its parameters, and a fixed smooth function $R(w)$ evaluates the final parameters after all
tasks have been learned.

Specifically, task $t\in[0,1]$ consists of an input matrix $X_t\in\R^{m\times p}$ and a target
vector $y_t\in\R^m$. Let $f_w(X_t)\in\R^m$ denote the predictions of the model with parameters
$w$, and define the training loss
$L_t(w)=\frac{1}{2m}\bigl\|f_w(X_t)-y_t\bigr\|^2.$
Assume that the data and model have three continuous derivatives. The exact-fit set of task $t$ is
$M_t=\{w:f_w(X_t)=y_t\}.$
For a linear model, $M_t$ is an affine space. For a nonlinear model, it is locally a smooth manifold
where the Jacobian of the prediction vector with respect to the parameters has full row rank.
This permits the analysis of moving zero-loss spaces to extend to nonlinear models.

Starting from an exact fit $a\in M_0$ of the first task, call the process that learns tasks
in increasing order of $t$ the \textbf{forward process}. Run gradient flow to an exact fit on each
task in this order and let the task discretization become dense. Denote the limiting final
parameter by $b$. The corresponding \textbf{reverse process} starts at $b$ and learns the same
family of tasks in decreasing order of $t$. Reversing the task path and exchanging the paired initial points swaps the roles of the two processes.

Divide $[0,1]$ into $N$ equal intervals and, in the chosen direction, learn the task at the end of
each interval. For direction $d\in\{+,-\}$, let $w_N^d(s)$ be the model parameters after learning
all tasks for time $s$ each, and write
\[
  F^d(\tau)=\lim_{N\to\infty}R\!\left(w_N^d(\tau/N)\right)
\]
for the continuum evaluation loss at fixed total training time $\tau$. In particular,
$F^+(0)=R(a),\quad F^-(0)=R(b).$

\begin{theorem}[Optimal training time for paired processes]
\label{thm:paired-orders}
In the setting above, assume the following.
\begin{enumerate}[label=(\arabic*)]
  \item \textbf{Nonsingularity.} As the task discretization becomes dense, the parameters obtained
  by sequential exact fitting converge to a continuous parameter path on $[0,1]$, and
  $D_wf_w(X_t)$ has full row rank along this path. The continuous training processes in both directions are defined
  for every finite total time $\tau$.
  \item \textbf{Strict improvement.} For each direction $d\in\{+,-\}$, there exists $\tau_d>0$ such that
  $F^d(\tau_d)<F^d(0).$
\end{enumerate}
Then at least one of the two directions $d$ has the following properties.
\begin{enumerate}[label=(\roman*)]
  \item For all sufficiently large $N$, a per-task training time minimizing the final evaluation
  loss $R(w_N^d(s))$ exists, and every optimizer satisfies
  $s_N^\star=\Theta(N^{-1}).$
  \item There exists a constant $\tau_R>0$, independent of $N$, such that
  \[
    R\!\left(w_N^d(\tau_R/N)\right)-\min_{s\ge0}R\!\left(w_N^d(s)\right)\longrightarrow0.
  \]
\end{enumerate}
\end{theorem}
Counting both directions of each pair satisfying the assumptions, the $1/N$ law holds for at least half of the paired processes.

\paragraph{Proof idea.}
The key technical step is to show that the two processes exchange endpoints as tasks become dense
and total training time diverges. \Cref{app:paired-proof} extends the moving-nullspace analysis to
smooth exact-fit manifolds, controlling accumulated tracking errors to show that the forward
process follows a path from $a$ to $b$, while the reverse process retraces it to $a$.

Choose the direction whose initial evaluation loss is $\min\{R(a),R(b)\}$.
Strict improvement gives a finite positive total time with loss below both $R(a)$ and $R(b)$,
beating the zero- and infinite-budget limits. The optimal total times therefore stay bounded and
away from zero. Uniform convergence on this interval yields the $1/N$ law and the asymptotic
optimality of a fixed total time. Full estimates are in \cref{app:paired-proof}.

For the counterexample in \cref{sec:risk-counterexample}, \cref{thm:paired-orders} guarantees $1/N$ scaling in the reverse order: its starting loss $R(e_1)$ is lower than the forward order's $R(u(0))$, and both orders satisfy strict improvement (\cref{app:risk-order-example}).

\subsection{Experiments}
\label{sec:numerical}

We test the $1/N$ scaling on Rotated MNIST and Rotated Fashion-MNIST by varying the number of
intermediate tasks along a fixed rotation path. Across the models, training losses, and evaluation
criteria listed in \cref{tab:budget-sweep}, we ask whether the optimal total budget stays
approximately constant and whether a budget tuned at one task count predicts the optimum at another.

\textbf{At small learning rates, the optimal total training time $N\eta k^\star$ is nearly constant
across task counts.} To test the theory's large-$N$ prediction, we fit $k^\star\propto N^{-b}$ over $N=40$--512.
Across all 30 settings at the three smallest learning rates, $b$ ranges from 0.95 to 1.01 (\cref{tab:budget-sweep}), while $N\eta k^\star$ increases by a median 5.38\% over this range.

\textbf{The scaling also predicts how to adjust the per-task budget.} Increasing the task count from
64 to 512 gives a simple prediction: use one-eighth as many updates per task. Across the same 30
settings, this prediction falls within 8\% of the fitted optimal update count at 512 tasks, with a
median relative error of 4\%.

\begin{figure}[!t]
\centering
\resizebox{\textwidth}{!}{\input{figures/budget_curves.tex}}
\caption{Optimal total training time $N\eta k^\star$ versus task count $N$. Rows select
$k^\star$ by held-out path loss or accuracy. The dashed line marks the start of the fit at $N=40$.
Circles mark interior optima. Triangles and crosses mark boundary-limited estimates unsuitable
for assessing the scaling law; see \cref{app:budget-transfer} for details.}
\label{fig:budget-curves}

\medskip

\centering
\captionof{table}{Estimated exponent $b$ in $k^\star\propto N^{-b}$ over $N=40$--512.
Brackets give 95\% bootstrap intervals over five seeds. A slash marks settings reaching one
update per task within the fit range, where the integer grid cannot resolve further decreases.
All ten configurations also reach this floor at $\eta=0.1,0.05,0.025$. Group headings list the model, dataset, and training loss.}
\label{tab:budget-sweep}
\begingroup
\normalsize
\setlength{\tabcolsep}{1pt}
\newcommand{\estimate}[3]{\textbf{#1}\,[#2,#3]}
\newcommand{\floorcell}{\textcolor{black!45}{/}}
\begin{tabular*}{\linewidth}{@{\extracolsep{\fill}}l*{5}{c}@{}}
\toprule
 & \multicolumn{2}{c}{Larger learning rates} & \multicolumn{3}{c}{Small learning rates}\\
\cmidrule(lr){2-3}\cmidrule(l){4-6}
\textbf{Eval.} & \textbf{0.0125} & \textbf{0.00625} & \textbf{0.003125} & \textbf{0.0015625} & \textbf{0.00078125}\\
\midrule
\multicolumn{6}{l}{MLP / Rotated MNIST / CE}\\
Loss & \floorcell & \floorcell & \estimate{0.98}{0.97}{0.99} & \estimate{0.98}{0.97}{0.98} & \estimate{0.98}{0.97}{0.98} \\
Accuracy & \floorcell & \floorcell & \estimate{0.98}{0.97}{0.99} & \estimate{0.97}{0.97}{0.99} & \estimate{0.97}{0.97}{0.98} \\
\midrule
\multicolumn{6}{l}{MLP / Rotated MNIST / Squared}\\
Loss & \estimate{0.97}{0.90}{1.03} & \estimate{0.96}{0.93}{1.01} & \estimate{0.95}{0.85}{0.99} & \estimate{0.97}{0.84}{0.98} & \estimate{0.96}{0.91}{0.97} \\
Accuracy & \floorcell & \estimate{0.97}{0.96}{0.99} & \estimate{0.99}{0.97}{1.00} & \estimate{0.98}{0.98}{0.99} & \estimate{0.98}{0.97}{0.99} \\
\midrule
\multicolumn{6}{l}{MLP / Rotated Fashion-MNIST / CE}\\
Loss & \floorcell & \floorcell & \estimate{0.99}{0.97}{1.02} & \estimate{1.00}{0.98}{1.01} & \estimate{0.97}{0.96}{0.98} \\
Accuracy & \floorcell & \floorcell & \estimate{0.99}{0.97}{1.02} & \estimate{0.98}{0.96}{0.99} & \estimate{0.98}{0.96}{0.99} \\
\midrule
\multicolumn{6}{l}{MLP / Rotated Fashion-MNIST / Squared}\\
Loss & \estimate{1.05}{0.08}{1.06} & \estimate{1.01}{0.48}{1.16} & \estimate{1.01}{0.72}{1.31} & \estimate{1.01}{0.91}{1.39} & \estimate{1.00}{0.91}{1.40} \\
Accuracy & \floorcell & \estimate{1.00}{0.98}{1.01} & \estimate{0.98}{0.96}{0.99} & \estimate{0.98}{0.97}{0.99} & \estimate{0.99}{0.97}{0.99} \\
\midrule
\multicolumn{6}{l}{Frozen trunk / Rotated MNIST / CE}\\
Loss & \estimate{0.99}{0.97}{1.00} & \estimate{0.98}{0.98}{0.99} & \estimate{0.98}{0.98}{0.99} & \estimate{0.98}{0.98}{0.99} & \estimate{0.98}{0.98}{0.99} \\
Accuracy & \estimate{0.99}{0.97}{1.01} & \estimate{0.98}{0.98}{0.99} & \estimate{0.98}{0.97}{0.99} & \estimate{0.98}{0.98}{0.99} & \estimate{0.99}{0.98}{0.99} \\
\bottomrule
\end{tabular*}
\endgroup

\end{figure}

\section{Discussion}

In the linear setting, very short total training barely changes the model, while very long total training tracks the moving zero-loss space with vanishing progress toward the common solution. A finite total budget avoids both limits. The theory extends the guarantee to smooth nonlinear models and arbitrary smooth evaluation losses, while the experiments show that the law predicts optimal training times accurately enough to guide tuning across different models, training losses, and evaluation criteria. Thus, when sampling tasks more frequently along the same transition, keeping the total training time approximately constant is a natural first choice to test. The theory guarantees the law for at least one process in each pair, while the counterexample rules out requiring it to hold for every process; empirically, the plateau depends on the task and training regime and remains less clear on Yearbook (\cref{app:yearbook}). It remains to determine whether the small upward drift in optimal total budgets diminishes at larger task counts, and how the motion of solution sets interacts with effects such as gradient noise and finite learning rates to determine those budgets.

\clearpage
\subsection*{AI use statement}

The authors conceived the question and developed the theory and methodology.  Generative AI assisted
with proof development and checking, experiment-code development and validation, literature search,
and manuscript drafting and editing.
The authors verified all AI-assisted content and take full responsibility for the paper.

\subsection*{Reproducibility statement}

The appendix gives complete proofs and experimental details.  The supplement provides code, fixed
seeds, CSV results, and figure sources for all reported experiments.

\bibliographystyle{iclr2027_conference}
\bibliography{references}

\clearpage

\appendix

\section{Proofs and Technical Details}
\label{app:technical}

\paragraph{Contribution and relation to existing mathematical tools.}
Our contribution is to explain how per-task training time determines progress toward a common
solution as a fixed task path is refined, and to establish the resulting optimal-time scaling law.
\Cref{thm:continuum-main} uses matrix-exponential expansions and contractive stability to connect
discrete task sequences to a continuum learning curve controlled by the total training time
$\tau=Ns_N$.  The large-budget analysis builds on Kato's parallel transport
and the adiabatic expansions for contracting evolutions of
\citet{avron2012contracting}, translating tracking of the moving zero-loss space into
$\Theta(\tau^{-1})$ decay of learning progress with an explicit path-dependent coefficient.
Building on this analysis, \cref{thm:optimal-main} uses a discrete tail bound uniform in the
task count and per-task training time to rule out optimal total training times tending to zero
or infinity.  This turns the two tails of the continuum curve into the discrete optimal-time
law $s_N^\star=\Theta(N^{-1})$. \Cref{app:budget-correction} shows that using a single total
budget $\tau_\star$ costs only $o(1/N)$ in progress relative to retuning at each $N$.
\Cref{prop:long-training} further quantifies residual accumulation
in discrete updates through a first-order coefficient with weights $\frac12\coth(s\lambda/2)$.
For every fixed finite training time, this coefficient is strictly larger than that of exact
fitting, yielding greater distance-based progress for sufficiently fine task partitions.
\Cref{app:prediction-risk,app:small-positive-curvature} clarify, respectively, the connection
to average prediction risk and the scope of the result when exact zero curvature is perturbed
by small positive curvature. \Cref{app:risk-order-example} gives a linear regression path on
which the total budget minimizing average population squared loss diverges.
\Cref{app:paired-proof} proves that, under the conditions of \cref{thm:paired-orders}, at least
one of the forward and reverse orders admits an asymptotically optimal fixed total budget.

\subsection{Pathwise loss and common-solution distance}
\label{app:pathwise-risk-equivalence}

\begin{proof}[Proof of \cref{prop:pathwise-risk-equivalence}]
Write $v=w-w^\star=v^\parallel+v^\perp$ with $v^\parallel\in\cK$ and
$v^\perp\in\cK^\perp$.  Since $H(t)v^\parallel=0$ for every $t$,
\[
  \cL_t(w)=\frac12(v^\perp)^\top H(t)v^\perp.
\]
The spectral upper bound in \cref{eq:spectral-gap} gives the right-hand inequality in
\cref{eq:pathwise-risk-equivalence}.  For the lower bound, first observe that
$\ker\bar H=\cK$.  Indeed, if $u^\top\bar H u=0$, then the continuous nonnegative function
$t\mapsto u^\top H(t)u$ vanishes everywhere.  Positive semidefiniteness then gives
$H(t)u=0$ for all $t$, so $u\in\cK$; the reverse inclusion is immediate.  Hence $\bar H$ is
positive definite on $\cK^\perp$ and $c_H>0$.  Finally,
\[
  \mathcal{R}_{\mathrm{path}}(w)
  \ge \int_0^1\cL_t(w)\,dt
  =\frac12(v^\perp)^\top\bar H v^\perp
  \ge c_H\|v^\perp\|^2.
\]
Because $\|v^\perp\|=\dist(w,w^\star+\cK)$, this proves
\cref{eq:pathwise-risk-equivalence}; the convergence equivalence follows from its two bounds.
\end{proof}

\subsection{The dense-task limit}

For a total training budget $\tau\ge0$, let $x_\tau:[0,1]\to\R^d$ solve
\begin{equation}
  \dot x_\tau(t)=-\tau H(t)x_\tau(t),
  \qquad x_\tau(0)=v_0,
  \label{eq:continuum-flow}
\end{equation}
and define its learning progress by
\begin{equation}
  \cP_\tau(v_0):=\|v_0\|^2-\|x_\tau(1)\|^2.
  \label{eq:continuum-progress}
\end{equation}

The following lemma proves \cref{thm:continuum-main}, including its uniform $O(N^{-1})$ rates.

\begin{lemma}[Dense-task limit]
\label{thm:continuum}
Under the assumptions of \cref{sec:setup}, if $Ns_N\to\tau\in[0,\infty)$, then
\begin{equation}
  \Phi_{N,s_N}v_0\longrightarrow x_\tau(1),
  \qquad
  \cP_{N,s_N}(v_0)\longrightarrow\cP_\tau(v_0).
  \label{eq:continuum-limit}
\end{equation}
For every $T<\infty$,
\begin{align}
  \sup_{0\le\tau\le T}
  \left\|\Phi_{N,\tau/N}v_0-x_\tau(1)\right\|
  &=O_T(N^{-1}),
  \label{eq:appendix-continuum-state-rate}\\
  \sup_{0\le\tau\le T}
  \left|\cP_{N,\tau/N}(v_0)-\cP_\tau(v_0)\right|
  &=O_T(N^{-1}).
  \label{eq:appendix-continuum-progress-rate}
\end{align}
Moreover,
\begin{equation}
  \cP_\tau(v_0)
  =2\tau\int_0^1x_\tau(t)^\top H(t)x_\tau(t)\,dt.
  \label{eq:continuum-energy}
\end{equation}
\end{lemma}

\begin{proof}
Set $h=N^{-1}$ and $b_N=Ns_N$, and let $x_b$ denote the solution of
\cref{eq:continuum-flow} with coefficient $b$.  Fix $T<\infty$.  Taylor expansion of the matrix
exponential and the integral form of the ODE give, uniformly for $b\in[0,T]$ and $j=1,\ldots,N$,
\[
  \left\|e^{-hbH(t_j)}x_b(t_{j-1})-x_b(t_j)\right\|\le C_T h^2.
\]
Indeed, both terms equal
$x_b(t_{j-1})-hbH(t_j)x_b(t_{j-1})$ up to $O_T(h^2)$; this uses the $C^1$ regularity of $H$.
If $e_j=v_j-x_b(t_j)$, the exact task update is a contraction, so
\[
  \|e_j\|\le\|e_{j-1}\|+C_T h^2.
\]
Since $e_0=0$, summing over the $N$ steps gives
$\|\Phi_{N,b/N}v_0-x_b(1)\|\le C_T h$.  This proves
\cref{eq:appendix-continuum-state-rate}.  Both the discrete updates and the ODE flow are
contractions, and hence
\[
  \left|\cP_{N,b/N}(v_0)-\cP_b(v_0)\right|
  \le 2\|v_0\|\,\|\Phi_{N,b/N}v_0-x_b(1)\|,
\]
which proves \cref{eq:appendix-continuum-progress-rate}.  To compare two coefficients $b,c\ge0$,
subtract their ODEs and apply variation of constants.  Since the homogeneous flow is a contraction,
\[
  \|x_b(1)-x_c(1)\|
  \le |b-c|\int_0^1\|H(t)\|\,\|x_c(t)\|\,dt
  \le L\|v_0\|\,|b-c|.
\]
Taking $b=b_N$ and $c=\tau$ proves \cref{eq:continuum-limit} whenever $b_N\to\tau$.

Finally,
\[
  \frac{d}{dt}\|x_\tau(t)\|^2
  =-2\tau x_\tau(t)^\top H(t)x_\tau(t).
\]
Integration from $0$ to $1$ gives \cref{eq:continuum-energy}.
\end{proof}

\subsubsection{A rotating rank-one example}
\label{app:rotating-example}

In \cref{fig:partial-fitting-mechanism}, let $R(\theta)$ denote counterclockwise rotation in
the plane and set $\omega=\pi/2$, $u(t)=R(\omega t)e_1$, $H(t)=u(t)u(t)^\top$, and $v_0=e_2$.
The zero-loss line is spanned by $R(\omega t)e_2$, and the origin is the only common solution.
In the dense-task limit, exact fitting follows the unit arc $R(\omega t)e_2$. For a finite total
budget $\tau$, the solution of $\dot x_\tau=-\tau H(t)x_\tau$ in rotating coordinates gives
\[
  x_\tau(t)=R(\omega t)\exp\!\left(t
  \begin{bmatrix}-\tau&\omega\\-\omega&0\end{bmatrix}\right)e_2 .
\]
Maximizing $1-\|x_\tau(1)\|^2$ numerically gives $\tau^\star\approx2.19836$. The three panels
show the two trajectories at $t=1/3$, $2/3$, and $1$.

\subsection{Exact expansions of the limiting curve}
\label{app:tail-details}

The small-budget behavior is described in the ambient parameter coordinates.  Define the operator
on $V(0)$
\begin{equation}
  \cA_H
  :=P(0)\left(\int_0^1H(t)\,dt\right)P(0)\bigg|_{V(0)}.
  \label{eq:AH}
\end{equation}
For large budgets, the solution directions move with $t$, so their contributions are compared in a
common coordinate system.  Let $U(t)$ be Kato transport,
\begin{equation}
  \dot U(t)=[\dot P(t),P(t)]U(t),
  \qquad U(0)=I.
  \label{eq:kato}
\end{equation}
The generator is skew-symmetric, $U(t)$ is orthogonal, and
$P(t)U(t)=U(t)P(0)$ \citep{avron2012contracting}.  Write
$P_0=P(0)$ and $Q_0=I-P_0$.  In this fixed frame,
\begin{equation}
  U(t)^\top H(t)U(t)
  =\begin{pmatrix}0&0\\0&H_\perp(t)\end{pmatrix},
  \qquad
  U(t)^\top\dot U(t)
  =\begin{pmatrix}0&-B(t)^\top\\B(t)&0\end{pmatrix},
  \label{eq:block-frame}
\end{equation}
relative to $V(0)\oplus V(0)^\perp$.  Here $H_\perp(t)\succeq\mu I$ and
$B(t)=Q_0U(t)^\top\dot P(t)U(t)P_0$.  Define
\begin{align}
  \cG_H
  &:=\int_0^1B(t)^\top H_\perp(t)^{-1}B(t)\,dt
  \label{eq:GH-block}\\
  &=\int_0^1
    U(t)^\top\dot P(t)H(t)^\dagger\dot P(t)U(t)\,dt\bigg|_{V(0)}.
  \label{eq:GH-ambient}
\end{align}
\begin{lemma}[Large-budget tracking expansion]
\label{lem:continuous-tracking}
Let $y_\tau(t)=U(t)^\top x_\tau(t)=(p_\tau(t),q_\tau(t))$ in the decomposition
$V(0)\oplus V(0)^\perp$.  As $\tau\to\infty$,
\begin{align}
  \sup_{t\in[0,1]}\|p_\tau(t)-v_0\|&=O(\tau^{-1}),
  \label{eq:p-tracking}\\
  q_\tau(t)&=-\tau^{-1}H_\perp(t)^{-1}B(t)v_0+r_\tau(t),
  \label{eq:q-tracking}\\
  \int_0^1\|r_\tau(t)\|^2dt&=o(\tau^{-2}).
  \label{eq:r-tracking}
\end{align}
\end{lemma}

\begin{proof}
Using \cref{eq:continuum-flow,eq:block-frame}, the transported components satisfy
\begin{equation}
  \dot p_\tau=B(t)^\top q_\tau,
  \qquad
  \dot q_\tau=-B(t)p_\tau-\tau H_\perp(t)q_\tau,
  \qquad (p_\tau(0),q_\tau(0))=(v_0,0).
  \label{eq:block-ode}
\end{equation}
Let $E_\tau(t,r)$ denote the evolution operator generated by
$-\tau H_\perp(t)$.  The spectral gap gives
$\|E_\tau(t,r)\|\le e^{-\mu\tau(t-r)}$.  Variation of constants in the second equation of
\cref{eq:block-ode} gives
\[
  q_\tau(t)=-\int_0^tE_\tau(t,r)B(r)p_\tau(r)\,dr,
  \qquad
  \|q_\tau\|_\infty\le \frac{C}{\tau}\|p_\tau\|_\infty.
\]
The first equation of \cref{eq:block-ode} therefore implies
\[
  \|p_\tau\|_\infty
  \le \|v_0\|+\frac{C}{\tau}\|p_\tau\|_\infty.
\]
For all sufficiently large $\tau$, this bounds $p_\tau$ uniformly.  Returning to the two estimates
above gives $\|q_\tau\|_\infty=O(\tau^{-1})$ and then
$\|p_\tau-v_0\|_\infty=O(\tau^{-1})$, proving \cref{eq:p-tracking}.

Set $a_\tau(t)=H_\perp(t)^{-1}B(t)p_\tau(t)$ and
$\widetilde r_\tau=q_\tau+\tau^{-1}a_\tau$.  \Cref{eq:block-ode} gives
\[
  \dot{\widetilde r}_\tau
  +\tau H_\perp(t)\widetilde r_\tau
  =\tau^{-1}\dot a_\tau(t),
  \qquad
  \widetilde r_\tau(0)=\tau^{-1}a_\tau(0).
\]
Smoothness, \cref{eq:p-tracking}, and the first equation of \cref{eq:block-ode} bound
$\dot a_\tau$ uniformly.  A second use of variation of constants gives
\[
  \widetilde r_\tau(t)
  =E_\tau(t,0)\frac{a_\tau(0)}{\tau}
   +\frac1\tau\int_0^tE_\tau(t,r)\dot a_\tau(r)\,dr.
\]
The first term has squared $L^2$ norm $O(\tau^{-3})$.  The second is
$O(\tau^{-2})$ uniformly in $t$ and hence has squared $L^2$ norm $O(\tau^{-4})$.
Finally, \cref{eq:p-tracking} gives
\[
  \frac1\tau H_\perp(t)^{-1}B(t)(p_\tau(t)-v_0)=O(\tau^{-2})
\]
uniformly.  Taking
$r_\tau=\widetilde r_\tau-\tau^{-1}H_\perp^{-1}B(p_\tau-v_0)$ proves
\cref{eq:q-tracking,eq:r-tracking}.
\end{proof}

We next give an estimate that is uniform in both task spacing and per-task training time.  Its
fixed-time expansions prove \cref{prop:long-training}, and its first bound is used to prove the
large-$Ns_N$ regime of \cref{thm:optimal-main}.

\subsubsection{Dense-task expansion for long per-task training}
\label{app:finite-time-proof}

\paragraph{First-order comparison at fixed per-task time.}
The main theorem concerns $s_N=\Theta(N^{-1})$, for which progress remains of constant order.  A
different question is what happens when the per-task time is fixed as $N$ grows.  Let
$C_s(v_0)=2\langle v_0,\cG_sv_0\rangle$ and
$C_H(v_0)=2\langle v_0,\cG_Hv_0\rangle$, where $\cG_H$ is defined in \cref{eq:GH-block} and $\cG_s$ is defined below.

\begin{proposition}[First-order comparison in the overtraining regime]
\label{prop:long-training}
For $v_0\in V(0)\cap\cK^\perp$ and every fixed $s\in(0,\infty]$,
\begin{equation}
  \cP_{N,s}(v_0)
  =\frac{C_s(v_0)}{N}+o(N^{-1}).
  \label{eq:main-fixed-s}
\end{equation}
If $s_N\downarrow0$ and $Ns_N\to\infty$, then
\begin{equation}
  \cP_{N,s_N}(v_0)
  =\frac{C_H(v_0)}{Ns_N}+o((Ns_N)^{-1}).
  \label{eq:main-small-s-long-time}
\end{equation}
Moreover, if $v_0\ne0$, then $C_s(v_0)>C_\infty(v_0)$ for every fixed finite $s$. Hence
$\cP_{N,s}(v_0)>\cP_{N,\infty}(v_0)$ for all sufficiently large $N$.
\end{proposition}

The strict comparison has a direct interpretation.  Along a positive-eigenvalue direction of
$H(t)$ with eigenvalue $\lambda$, finite training leaves the fraction $\rho=e^{-s\lambda}$ of the incoming
error.  Its contribution to the first-order coefficient can be written as
\begin{equation}
  \frac12\coth(s\lambda/2)
  =\frac12+\frac{\rho}{1-\rho}.
  \label{eq:main-residual-weight}
\end{equation}
Exact fitting sets $\rho=0$.  Finite fitting adds the positive second term, which records residual
error that later, slightly different tasks can still reduce.  Thus fixed finite fitting and exact
fitting both have progress of order $N^{-1}$, but the finite-time coefficient is strictly larger.
We now define the coefficients and prove these claims.

For $s\in(0,\infty)$, define on the normal block
\begin{equation}
  \cW_s(t):=\frac12\coth\!\left(\frac{sH_\perp(t)}2\right),
  \qquad
  \cW_\infty(t):=\frac12I,
  \label{eq:appendix-Ws}
\end{equation}
and set
\begin{equation}
  \cG_s:=\int_0^1B(t)^\top\cW_s(t)B(t)\,dt.
  \label{eq:appendix-Gs}
\end{equation}
Together with $\cG_H$ from \cref{eq:GH-block}, these are the operators used in the coefficients of
\cref{prop:long-training}.

\begin{lemma}[Finite-time expansion]
\label{lem:discrete-tracking}
Let
\[
  \delta_N:=1-e^{-\mu s_N}.
\]
If $N\delta_N\to\infty$, then
\begin{equation}
  \cP_{N,s_N}(v_0)=O\!\left(\frac{1}{N\delta_N}\right).
  \label{eq:uniform-discrete-bound}
\end{equation}
For every fixed $s\in(0,\infty]$,
\begin{align}
  P(1)\Phi_{N,s}\big|_{V(0)}
  &=U(1)\left(I-\frac{\cG_s}{N}\right)+o(N^{-1}),
  \label{eq:appendix-fixed-s-operator}\\
  \cP_{N,s}(v_0)
  &=\frac{2}{N}\langle v_0,\cG_sv_0\rangle+o(N^{-1}).
  \label{eq:appendix-fixed-s-progress}
\end{align}
If additionally $s_N\to0$, then
\begin{equation}
  \cP_{N,s_N}(v_0)
  =\frac{2}{Ns_N}\langle v_0,\cG_Hv_0\rangle
   +o((Ns_N)^{-1}).
  \label{eq:discrete-large-tail}
\end{equation}
Moreover, for every fixed finite $s$ and every nonzero
$v_0\in V(0)\cap\cK^\perp$,
\begin{equation}
  \langle v_0,(\cG_s-\cG_\infty)v_0\rangle>0.
  \label{eq:strict-fixed-vs-exact}
\end{equation}
\end{lemma}

\begin{proof}
Let $h=N^{-1}$, $U_j=U(t_j)$, and
$z_j=U_j^\top v_j=(p_j,q_j)$.  Put $B_j=B(t_{j-1})$.  Taylor expansion of
$U_j^\top U_{j-1}$ in the block frame \cref{eq:block-frame} gives
\begin{align}
  p_j&=p_{j-1}+hB_j^\top q_{j-1}
       -\frac{h^2}{2}B_j^\top B_jp_{j-1}
       +a_j,
  \label{eq:discrete-p}\\
  q_j&=R_j(q_{j-1}-hB_jp_{j-1})
       +b_j,
  \label{eq:discrete-q}
\end{align}
where $R_j=e^{-s_NH_\perp(t_j)}$ satisfies
$\|R_j\|\le1-\delta_N$, while
\[
  \|a_j\|\le C(h^3+h^2\|q_{j-1}\|),
  \qquad
  \|b_j\|\le Ch^2.
\]
These bounds are uniform in both $j$ and $s_N$.  Indeed, the block generator
$S(t):=U(t)^\top\dot U(t)$ in \cref{eq:block-frame} has vanishing diagonal blocks, while
$S(t)^2$ has tangent block $-B(t)^\top B(t)$.  A second-order Taylor expansion of
$U_j^\top U_{j-1}$ therefore gives tangent--normal blocks $hB_j^\top+O(h^2)$ and
$-hB_j+O(h^2)$, and tangent block
$I-\frac{h^2}{2}B_j^\top B_j+O(h^3)$.  The $C^3$ bounds on $U$ make these remainders uniform;
multiplication of the normal block by $R_j$, with $\|R_j\|\le1$, preserves the stated estimates.
The original updates are contractions, so $p_j$ and $q_j$ remain bounded.  Iterating
\cref{eq:discrete-q} and using $h/\delta_N=(N\delta_N)^{-1}\to0$ yields
\[
  \max_{j\le N}\|q_j\|=O\!\left(\frac{h}{\delta_N}\right).
\]
Substitution into \cref{eq:discrete-p}, followed by discrete Gronwall, gives
\[
  \max_{j\le N}\|p_j-v_0\|
  =O\!\left(\frac{h}{\delta_N}\right).
\]
Because $U_N$ is orthogonal,
\[
  \cP_{N,s_N}(v_0)
  =\|v_0\|^2-\|p_N\|^2-\|q_N\|^2.
\]
The last two bounds, together with $h/\delta_N\to0$, prove
\cref{eq:uniform-discrete-bound}.

Now fix $s\in(0,\infty)$, so that $R_j=e^{-sH_\perp(t_j)}$ is uniformly contractive.  Define
\begin{equation}
  L_j:=-(I-R_j)^{-1}R_jB_j,
  \qquad
  W_j:=\frac12(I+R_j)(I-R_j)^{-1}.
  \label{eq:fixed-s-LW}
\end{equation}
Smoothness gives $\|L_j\|\le C$ and $\|L_j-L_{j-1}\|\le Ch$.  For
$0\le j\le N-1$, set $e_j=q_j-hL_{j+1}p_j$.  The identity
$(I-R_j)L_j=-R_jB_j$ and
\cref{eq:discrete-p,eq:discrete-q} imply
\[
  \|e_j\|\le(1-\delta)\|e_{j-1}\|+Ch^2,
  \qquad \delta:=1-e^{-\mu s}>0.
\]
Since $e_0=O(h)$, the geometric-series bound gives
$\|e_j\|\le Ch(1-\delta)^j+Ch^2$ and hence
$h\sum_{j=1}^N\|e_{j-1}\|=O(h^2)=o(h)$.  Also
$q_{j-1}=hL_jp_{j-1}+e_{j-1}$.  Substitution into \cref{eq:discrete-p} gives
\begin{equation}
  p_N
  =\left[I-h\sum_{j=1}^NhB_j^\top W_jB_j\right]v_0+o(h),
  \qquad q_N=O(h).
  \label{eq:fixed-s-slow-update}
\end{equation}
On the normal block,
\[
  W_j=\frac12\coth\!\left(\frac{sH_\perp(t_j)}2\right).
\]
The Riemann sum in \cref{eq:fixed-s-slow-update} therefore converges to $\cG_s$ from
\cref{eq:appendix-Gs}.  Since $P(1)v_N=U(1)p_N$ and $\|q_N\|^2=o(h)$, expansion of the endpoint norm
proves \cref{eq:appendix-fixed-s-operator,eq:appendix-fixed-s-progress}.  The case $s=\infty$
follows by setting $R_j=0$, equivalently $W_j=I/2$.

For every positive eigenvalue $\lambda$ of $H_\perp(t)$,
\[
  \frac12\coth(s\lambda/2)-\frac12
  =\frac{e^{-s\lambda}}{1-e^{-s\lambda}}>0.
\]
Hence $\cG_s-\cG_\infty$ is positive semidefinite.  Equality on $v_0$ would require
$B(t)v_0=0$ for every $t$.  To see why this forces $v_0\in\cK$, set
$z(t)=U(t)v_0\in V(t)$.  The vector $U(t)^\top\dot P(t)z(t)$ lies in $V(0)^\perp$ and equals
$B(t)v_0$ there.  The Kato equation therefore gives
$\dot z(t)=\dot P(t)z(t)=0$.  Hence $z(t)=v_0\in V(t)$ for every $t$, so
$v_0\in\cK$.  This contradicts nonzero $v_0\in\cK^\perp$ and proves
\cref{eq:strict-fixed-vs-exact}.

For the refined expansion, suppose $s_N\to0$.  Then $\delta_N\asymp s_N$.  Define
\[
  L_j:=-(I-R_j)^{-1}R_jB_j,
  \qquad
  W_j:=\frac12(I+R_j)(I-R_j)^{-1}.
\]
The spectral bounds and smoothness give
\[
  \|L_j\|\le Cs_N^{-1},
  \qquad
  \|L_j-L_{j-1}\|\le Chs_N^{-1}.
\]
Set $e_j=q_j-hL_{j+1}p_j$ for $0\le j\le N-1$.  The identity
$(I-R_j)L_j=-R_jB_j$, together with \cref{eq:discrete-p,eq:discrete-q}, gives
\[
  \|e_j\|
  \le (1-\delta_N)\|e_{j-1}\|
     +C\frac{h^2}{s_N}
     +C\frac{h^3}{s_N^2}.
\]
Since $e_0=O(h/s_N)$, a geometric-series estimate yields
\[
  \|e_j\|
  \le C\frac{h}{s_N}(1-\delta_N)^j
     +C\frac{h^2}{s_N^2}.
\]
Since $\delta_N\asymp s_N$ and the standing assumption gives
$h/s_N=(Ns_N)^{-1}\to0$, summing this bound makes the little-$o$ step explicit:
\[
  h\sum_{j=1}^N\|e_{j-1}\|
  \le C\frac{h^2}{s_N\delta_N}+C\frac{h^2}{s_N^2}
  =O\!\left(\frac{h^2}{s_N^2}\right)
  =o\!\left(\frac{h}{s_N}\right).
\]
Substituting $q_{j-1}=hL_jp_{j-1}+e_{j-1}$ into \cref{eq:discrete-p} and summing gives
\[
  p_N
  =\left[I-h\sum_{j=1}^NhB_j^\top W_jB_j\right]v_0
   +o\!\left(\frac{h}{s_N}\right).
\]
Finally, uniformly on the positive spectrum,
\[
  W_j
  =\frac12I+(I-e^{-s_NH_\perp(t_j)})^{-1}e^{-s_NH_\perp(t_j)}
  =s_N^{-1}H_\perp(t_j)^{-1}+O(s_N).
\]
The Riemann sum therefore gives
$p_N=v_0-(Ns_N)^{-1}\cG_Hv_0+o((Ns_N)^{-1})$.
Since $\|q_N\|^2=O((Ns_N)^{-2})=o((Ns_N)^{-1})$, expansion of
$\|p_N\|^2+\|q_N\|^2$ proves \cref{eq:discrete-large-tail}.
\end{proof}

\subsection{An anisotropic distance can reverse the comparison}
\label{app:anisotropic-counterexample}

\paragraph{Robustness to monotone distance transformations.}
This conclusion does not depend on using squared distance itself.  For any strictly increasing
$f:[0,\infty)\to\R$, define
$\cP^f_{N,s}(v_0):=f(\|v_0\|^2)-f(\|\Phi_{N,s}v_0\|^2)$.  Writing
$r=\|v_0\|^2$ gives the exact identity
\begin{equation}
  \cP^f_{N,s}(v_0)
  =f(r)-f\bigl(r-\cP_{N,s}(v_0)\bigr).
  \label{eq:transformed-progress-identity}
\end{equation}
The right-hand side is strictly increasing in $\cP_{N,s}$, so signs, comparisons across training
times, and maximizers are unchanged.  The conclusion need not hold for anisotropic distances.  An
explicit reversal follows.

For a positive-definite matrix $M$, define the corresponding endpoint progress by
\[
  \cP^M_{N,s}(v)
  :=v^\top Mv-(\Phi_{N,s}v)^\top M(\Phi_{N,s}v).
\]
The following closed four-dimensional path shows that the fixed-finite-versus-exact comparison need
not survive this change of geometry.  Let
\[
  P_0=\operatorname{diag}(I_2,0_2),
  \qquad
  \Omega=2\pi I_2,
  \qquad
  J=\begin{pmatrix}0&-\Omega^\top\\ \Omega&0\end{pmatrix},
\]
and define
\[
  R(t)=e^{tJ},
  \qquad
  P(t)=R(t)P_0R(t)^\top,
  \qquad
  H(t)=R(t)H_0R(t)^\top,
\]
where
\[
  H_0=\operatorname{diag}\!\left(0,0,\log2,\log\frac54\right).
\]
Take
\[
  v=\frac{1}{\sqrt5}(-2,1,0,0)^\top,
  \qquad
  M=\operatorname{diag}(M_0,I_2),
  \qquad
  M_0=\begin{pmatrix}1&0.9\\0.9&1\end{pmatrix}.
\]
The eigenvalues of $M_0$ are $0.1$ and $1.9$, so $M$ is positive definite.  The path is closed
because $R(1)=I$.  Since $J$ is off-diagonal relative to $P_0$, direct calculation gives
$[\dot P(t),P(t)]=J$; hence its Kato transport is $U(t)=R(t)$.  In the moving frame, $B(t)=\Omega$ and
$H_\perp=\operatorname{diag}(\log2,\log(5/4))$.  Therefore, at fitting time $s=1$,
\begin{equation}
  D:=\cG_1-\cG_\infty
  =(2\pi)^2\operatorname{diag}(1,4).
  \label{eq:counterexample-D}
\end{equation}

Because $U(1)=I$ and $M$ is block diagonal, the tangent--normal cross term vanishes.  The normal
endpoint component is $O(N^{-1})$ by \cref{eq:fixed-s-slow-update}, so its quadratic contribution is
$o(N^{-1})$.  The endpoint expansion \cref{eq:appendix-fixed-s-operator} therefore gives
\[
  \lim_{N\to\infty}
  N\bigl(\cP^M_{N,1}(v)-\cP^M_{N,\infty}(v)\bigr)
  =\widehat v^\top(M_0D+DM_0)\widehat v.
\]
Here $\widehat v=(-2,1)^\top/\sqrt5$ is the tangent coordinate vector of $v$.
For Euclidean distance, $M_0=I_2$, and the right-hand side is
\begin{equation}
  \frac{16}{5}(2\pi)^2>0.
  \label{eq:counterexample-euclidean}
\end{equation}
For the positive-definite matrix above,
\[
  M_0\operatorname{diag}(1,4)
  +\operatorname{diag}(1,4)M_0
  =\begin{pmatrix}2&4.5\\4.5&8\end{pmatrix},
\]
so the limit is instead
\begin{equation}
  -\frac25(2\pi)^2<0.
  \label{eq:counterexample-anisotropic}
\end{equation}
Thus finite fitting beats exact fitting under Euclidean distance but loses under this anisotropic
distance on the same task path and initial direction.

\subsection{Proof of the optimal training-time law}

The following lemma proves parts (i)--(iii) of \cref{thm:optimal-main} and records the coefficients of its two tails.

\begin{lemma}[Optimal-time law]
\label{thm:optimal-time}
Let $v_0\in V(0)\cap\cK^\perp$ be nonzero.  Then:
\begin{enumerate}[label=(\roman*)]
  \item The limiting progress has the two asymptotic expansions
  \begin{align}
    \cP_\tau(v_0)
    &=2\tau\langle v_0,\cA_Hv_0\rangle+O(\tau^2),
    &&\tau\downarrow0,
    \label{eq:small-tail}\\
    \cP_\tau(v_0)
    &=\frac{2}{\tau}\langle v_0,\cG_Hv_0\rangle+o(\tau^{-1}),
    &&\tau\to\infty.
    \label{eq:large-tail}
  \end{align}
  \item The two tail operators satisfy
  \begin{equation}
    \ker\cA_H=\ker\cG_H=\cK
    \quad\text{within }V(0),
    \label{eq:operator-kernels}
  \end{equation}
  and $\cP_\tau(v_0)>0$ for every $\tau\in(0,\infty)$.  Hence
  \begin{equation}
    \cT^\star(v_0):=\argmaxop_{\tau>0}\cP_\tau(v_0)
    \label{eq:limit-optimizer-set}
  \end{equation}
  is a nonempty compact subset of $(0,\infty)$.
  \item For every sequence $s_N$,
  \begin{equation}
  \begin{array}{rcl}
    Ns_N\to0
    &\Longrightarrow&\cP_{N,s_N}(v_0)\to0,\\[1mm]
    Ns_N\to\tau\in(0,\infty)
    &\Longrightarrow&\cP_{N,s_N}(v_0)\to\cP_\tau(v_0)>0,\\[1mm]
    Ns_N\to\infty
    &\Longrightarrow&\cP_{N,s_N}(v_0)\to0.
  \end{array}
  \label{eq:three-regimes}
  \end{equation}
  \item For all sufficiently large $N$, every optimizer in \cref{eq:finite-optimizer} is finite.
  Every sequence of such optimizers satisfies
  \begin{equation}
    \dist\!\left(Ns_N^\star(v_0),\cT^\star(v_0)\right)\longrightarrow0.
    \label{eq:optimizer-localization}
  \end{equation}
  Consequently $s_N^\star(v_0)=\Theta(N^{-1})$.  If
  $\cT^\star(v_0)=\{\tau^\star(v_0)\}$, then
  \begin{equation}
    s_N^\star(v_0)=\frac{\tau^\star(v_0)}{N}+o(N^{-1}).
    \label{eq:unique-optimum}
  \end{equation}
\end{enumerate}
\end{lemma}

\begin{proof}
For small $\tau$, the integral form of \cref{eq:continuum-flow} gives, uniformly on $[0,1]$,
\[
  x_\tau(t)
  =v_0-\tau\int_0^tH(r)v_0\,dr+O(\tau^2).
\]
Since $x_\tau(t)=v_0+O(\tau)$ uniformly, substitution into
\cref{eq:continuum-energy} proves \cref{eq:small-tail}.  For large $\tau$,
\cref{lem:continuous-tracking} and \cref{eq:block-frame} give
\[
  \cP_\tau(v_0)
  =2\tau\int_0^1q_\tau(t)^\top H_\perp(t)q_\tau(t)\,dt,
\]
Set $g(t)=H_\perp(t)^{-1}B(t)v_0$.  By
\cref{eq:q-tracking,eq:r-tracking}, $q_\tau=-\tau^{-1}g+r_\tau$ with
$\|r_\tau\|_{L^2}=o(\tau^{-1})$.  Because $g$ and $H_\perp$ are bounded,
\[
  \int_0^1q_\tau^\top H_\perp q_\tau\,dt
  =\frac1{\tau^2}\int_0^1g^\top H_\perp g\,dt+o(\tau^{-2}).
\]
The integral equals $\langle v_0,\cG_Hv_0\rangle$, which proves
\cref{eq:large-tail}.

For $v\in V(0)$,
\begin{align*}
  \langle v,\cA_Hv\rangle
  &=\int_0^1\|H(t)^{1/2}v\|^2dt,\\
  \langle v,\cG_Hv\rangle
  &=\int_0^1\|H_\perp(t)^{-1/2}B(t)v\|^2dt.
\end{align*}
By continuity, the first integral vanishes exactly for $v\in\cK$.  For the second, vanishing is
equivalent to $B(t)v=0$ for every $t$.  Let $z(t)=U(t)v\in V(t)$.  Since $\dot P(t)$ maps
$V(t)$ into $V(t)^\perp$, the normal coordinate of $\dot P(t)z(t)$ is exactly $B(t)v$.
The Kato equation gives $\dot z(t)=\dot P(t)z(t)$, so $B(t)v=0$ makes $z(t)$ constant and places
$v$ in every $V(t)$.  Conversely, if $v\in\cK$, then $P(t)v=v$ and differentiation gives
$\dot P(t)v=0$, so the constant path solves the Kato equation and $B(t)v=0$.  This proves
\cref{eq:operator-kernels}.

If $\cP_\tau(v_0)=0$ for a finite positive $\tau$, the nonnegative integrand in
\cref{eq:continuum-energy} vanishes for every $t$.  Hence $H(t)x_\tau(t)=0$ for every $t$, the ODE
gives $\dot x_\tau(t)=0$, and $v_0=x_\tau(t)$ belongs to every $V(t)$.  This contradicts nonzero
$v_0\in\cK^\perp$, so $\cP_\tau(v_0)>0$.  The curve is continuous, is positive at each finite
$\tau>0$, and tends to zero at both ends by \cref{eq:small-tail,eq:large-tail}.  Its maximum is
therefore attained on a compact subinterval of $(0,\infty)$, which proves the claim about
\cref{eq:limit-optimizer-set}.

The first two lines of \cref{eq:three-regimes} follow from \cref{thm:continuum}.  If
$Ns_N\to\infty$, use
\[
  1-e^{-\mu s}\ge c_\mu\min\{s,1\},\qquad s\ge0,
\]
for a constant $c_\mu>0$.  If $s_N\le1$, this gives
$N(1-e^{-\mu s_N})\ge c_\mu Ns_N$; if $s_N>1$, it gives a lower bound of order $N$.
Thus $N(1-e^{-\mu s_N})\to\infty$, and
\cref{lem:discrete-tracking} gives the third line.

It remains to localize the finite-resolution optimizers.  Choose any $\bar\tau>0$.  By
\cref{thm:continuum},
$\cP_{N,\bar\tau/N}(v_0)\to\cP_{\bar\tau}(v_0)>0$, so the optimal value stays bounded away from
zero.  If no common positive lower bound existed for $Ns_N^\star$, one could choose task counts and
optimizers with $Ns_N^\star\to0$; the first regime would make their progress tend to zero, a
contradiction.  If no common upper bound existed, one could similarly choose
$Ns_N^\star\to\infty$, including the possibility $s_N^\star=\infty$; the third regime gives the same
contradiction.  Hence all optimizers are finite and $Ns_N^\star$ lies in one compact subinterval of
$(0,\infty)$ for all sufficiently large $N$.

Consider any convergent subsequence $Ns_N^\star\to\tau$.  For each fixed $\gamma>0$, optimality gives
$\cP_{N,s_N^\star}(v_0)\ge\cP_{N,\gamma/N}(v_0)$.  Uniform convergence on the compact interval
containing $Ns_N^\star$, together with \cref{thm:continuum}, yields
\[
  \cP_\tau(v_0)\ge\cP_\gamma(v_0)
  \qquad\text{for every }\gamma>0.
\]
Every accumulation point therefore belongs to $\cT^\star(v_0)$.  If
\cref{eq:optimizer-localization} failed, a subsequence at a fixed positive distance from
$\cT^\star(v_0)$ would have a further convergent subsequence whose limit belongs to
$\cT^\star(v_0)$, a contradiction.  This proves \cref{eq:optimizer-localization}; the final two
claims follow.
\end{proof}

\subsubsection{First-order correction at a fixed total budget}
\label{app:budget-correction}
Put $h=1/N$, $A(t)=-\tau H(t)$, and let $U(t,u)$ propagate $\dot x=A(t)x$.
On bounded budget intervals the right-endpoint exponential minus the exact one-step propagator
is $h^2 A'(t)/2+O(h^3)$ at a step starting at $t$, by Taylor expansion.  Telescoping the product
and replacing the Riemann sum by its integral gives the global error operator
\[
 \frac h2\int_0^1 U(1,t)A'(t)U(t,0)\,dt+O(h^2)
 =\frac h2\{A(1)U(1,0)-U(1,0)A(0)\}+O(h^2).
\]
Indeed the integrand is the derivative of $U(1,t)A(t)U(t,0)$.
All remainders are uniform on compact budget intervals by smoothness and boundedness of the
propagators there.  Since $A(0)v_0=0$, writing $x=x_\tau(1)$ gives
$\cP_{N,\tau/N}=\cP_\tau+h B(\tau)+O(h^2)$ with $B(\tau)=\tau x^\top H(1)x$ (the argument $v_0$ is suppressed).
Choose $\tau_\star$ maximizing $B$ on the compact set of limiting maximizers.  For any discrete
optimizer $\tau_N=Ns_N^\star$, part (iii) implies $\limsup B(\tau_N)\le B(\tau_\star)$.
Since $\cP_{\tau_N}\le\cP_{\tau_\star}$, the nonnegative discrete progress gap divided by $h$ is at most
$B(\tau_N)-B(\tau_\star)+O(h)$.  This proves part (iv).

\input{appendix_risk_curvature}
\input{appendix_paired_budgets}

\clearpage
\subsection{Smoothness of the nullspace projector}
\label{app:nullspace-smoothness}

For a constant-rank symmetric matrix path, smooth variation of the nullspace means that its
orthogonal projector $P(t)$ varies smoothly. Here $H$ is $C^3$. Its ordered eigenvalues are
continuous, so for positive constant rank the smallest positive eigenvalue stays positive on
the compact interval $[0,1]$. With $\mu$ as in \cref{eq:spectral-gap}, the counterclockwise
circle $\Gamma=\{z\in\mathbb C:|z|=\mu/2\}$ separates zero from every positive eigenvalue
for all $t$. By the spectral theorem,
\[
  P(t)=\frac{1}{2\pi i}\oint_\Gamma (zI-H(t))^{-1}\,dz .
\]
Indeed, on an eigenspace with eigenvalue $\lambda$, the integrand acts by $(z-\lambda)^{-1}$.
By the residue theorem its normalized integral is one for $\lambda=0$ and zero for
$\lambda>0$. On the fixed contour the inverse exists and its first three derivatives in $t$
are continuous on $\Gamma\times[0,1]$. Differentiating under the integral gives $P\in C^3$.
In the rank-zero case $P(t)=I$ throughout.

\section{Details of the Numerical Experiments}
\label{app:numerical-details}

\subsection{Protocol for the rotation benchmarks}
\label{app:rotation-protocol}

\paragraph{Data and task path.}
For each seed, we sample 30,000 images without replacement from the standard training split and
3,000 from the standard test split of MNIST or Fashion-MNIST. These subsets remain fixed across
task counts, learning rates, and training budgets for that seed. Images have size $28\times28$.
Pixel values are divided by 255, rotated using bilinear interpolation with zero padding and no
canvas expansion, and then standardized. The mean and standard deviation are $(0.1307,0.3081)$
for MNIST and $(0.2860,0.3530)$ for Fashion-MNIST. Rotations use
\texttt{affine\_grid} and \texttt{grid\_sample} with \texttt{align\_corners=False}.

Task $j$ rotates each image by $60j/N$ degrees and preserves its label, for $j=1,\ldots,N$.
After the entire sequence, we evaluate the final model at the same 13 angles
$0^\circ,5^\circ,\ldots,60^\circ$ for every $N$. The path criterion is the unweighted average
over these angles. Increasing $N$ therefore refines the training path without changing the
evaluation grid. Evaluation images never enter gradient updates. We use them to study the
training-time optimum of a fixed empirical criterion, rather than to estimate predictive
performance on a separate post-selection test set.

\paragraph{Models, losses, and initialization.}
The MLP has fully connected layers of widths $784$--$256$--$128$--$10$, with biases and ReLU
activations after the two hidden layers. Training uses either cross-entropy on the ten logits
or squared loss $\frac12\|f_w(x)-e_y\|^2$, averaged over examples, where $e_y$ is the one-hot
label vector. Squared loss is applied to raw outputs, without a softmax. Accuracy uses the
largest output coordinate. In the loss-selected rows of \cref{tab:budget-sweep}, the held-out
loss is the same loss used for training; the accuracy-selected rows use held-out accuracy.

For each dataset, training loss, and seed, the source model is trained at $0^\circ$ for 1,000
updates with learning rate 0.05. Each adaptation run restarts from this source checkpoint.
The frozen-trunk configuration uses the same MLP on MNIST with cross-entropy. After source
training, the two hidden layers are frozen, the final linear layer is reinitialized with seed
one larger than the run seed, and this layer is refitted at $0^\circ$ for another 1,000 updates
at learning rate 0.05. Only the final layer is updated during adaptation. All source, refit,
and adaptation updates use plain mini-batch SGD with batch size 256, zero momentum, and no
weight decay. Larger per-task update counts use longer prefixes of the same mini-batch stream
within each $(N,\text{seed})$ pair. All reported runs use the CPU.

\paragraph{Sweep grid.}
We use seeds 20260805 through 20260809 and learning rates $\eta=0.1/2^j$, $j=0,\ldots,7$.
The 22 task counts are
\[
\begin{split}
N\in\{&4,5,6,8,10,13,16,20,25,32,40,51,64,81,102,128,\allowbreak
161,203,256,323,406,512\}.
\end{split}
\]
For each $N$ and $\eta$, nominal total times are spaced by a factor of $\sqrt2$.
For the full MLP, they range from 0.4 to 12.8 with cross-entropy and from 1 to 32 with squared
loss, giving 11 nominal times in each case. For the frozen trunk, the cross-entropy grid extends
from 0.4 to 25.6, giving 13 nominal times. A nominal time $\tau$ is converted to
$k=\max\{1,\operatorname{round}(\tau/(N\eta))\}$ updates per task, and duplicate counts are removed.
We also evaluate $k=0$ as the source baseline. The actual total time is $N\eta k$ after rounding.
The exact commands and seed-level results are in \texttt{experiments/rotation\_grid/} of the
supplement, with the run grid specified in \texttt{run\_grid.sh}.

\subsection{Estimating and transferring the optimal budget}
\label{app:budget-transfer}

\paragraph{Estimating the optimum.}
At each task count and learning rate, we average the held-out evaluation curves over the five
seeds. We estimate $k^\star$ separately using path loss and negative path accuracy, so both
criteria are minimized. The logarithmic fit uses positive update counts; the $k=0$ source
baseline is excluded. Around the best grid count, we fit a quadratic in $\log_2 k$ to that
point and its two neighbors. If the best point is at a boundary, we use the three nearest
grid points instead. When the fitted quadratic is convex, its minimizer is clipped to the
interval spanned by those three points. Otherwise, we retain the best grid count. With fewer
than three available counts, we also retain the grid optimum.

\paragraph{Boundary markers and scaling fits.}
The markers in \cref{fig:budget-curves} describe the location of the best grid count before
interpolation. Circles indicate an interior optimum. Crosses indicate that the best count is
one update per task, the smallest positive count allowed. Triangles indicate either of the
other sweep boundaries, where the optimum is not bracketed. These boundary estimates describe
the finite search grid and do not resolve the unconstrained training-time optimum.

We regress $\log k^\star$ on $\log N$ over the 12 task counts from 40 to 512 and report the
negative slope $b$. The analysis computes this regression over the full window; it does not
selectively drop boundary points. In \cref{tab:budget-sweep}, settings that reach the one-update
floor anywhere in the window are replaced by a slash. All numerically reported cells have
interior grid optima throughout this window. Confidence intervals use 2,000 bootstrap samples
of the five complete seed trajectories, recomputing the mean curves, interpolated optima, and
slope each time. The 2.5th and 97.5th percentiles give the reported interval. The bootstrap
generator is initialized with seed zero. This procedure measures variation across the five
seed runs, including their sampled data subsets, without separately resampling evaluation
examples within a run.

\paragraph{Transfer from 64 to 512 tasks.}
We hold the model, dataset, training loss, evaluation criterion, and learning rate fixed.
Keeping the total time constant predicts $\widehat k_{512}=k^\star_{64}/8$.
For each of the 30 settings at the three smallest learning rates, the reported relative error is
$|\widehat k_{512}-k^\star_{512}|/k^\star_{512}$, using interpolated optima rather than rounded
integer schedules. Its maximum is 7.83\% and its median is 3.70\%, giving the rounded values
in the main text. This checks prediction of the fitted optimum, rather than the downstream
loss of a separately evaluated rounded schedule.

The empirical plateaus are approximate: the prediction is below the fitted optimum in 27 of
the 30 settings, and the median increase in $N\eta k^\star$ from $N=40$ to 512 is 5.38\%.
Moreover, 24 of the 30 exponent intervals lie entirely below one. These observations support
accurate finite-range prediction without claiming an exact exponent of one in every setting.
The runs all follow increasing rotation angles from $0^\circ$ to $60^\circ$; they do not test
both paired orders or identify the infinite-budget endpoints required by
\cref{thm:paired-orders}. Thus the experiments test the scaling beyond the theorem's verified
conditions, rather than verify those conditions themselves.

\subsection{Temporal Yearbook}
\label{app:yearbook}

The Yearbook dataset contains aligned American high-school yearbook portraits indexed by calendar
year \citep{ginosar2017yearbooks}.  We use the official Wild-Time preprocessing and retain its
provided data split \citep{yao2022wildtime}: $33{,}431$ images are used for training and $3{,}758$
images from the benchmark test partition are held out for evaluation.  The images span 1930--2013
and are represented as $32\times32$ grayscale inputs.  We inherit the benchmark's binary labels only
to study temporal optimization; the experiment is not intended for demographic inference.

To keep the underlying path fixed as $N$ changes, a task centered at calendar year $t$ reweights
example $i$, with year $a_i$, according to
\[
  q_t(i)
  \propto
  \exp\!\left(-\frac{(a_i-t)^2}{2\sigma^2}\right)
  \mathbf{1}\{|a_i-t|\le4\sigma\},
  \qquad \sigma=4\ \text{years}.
\]
For $N\in\{4,8,16,32\}$, training visits the centers
$t_j=1930+83j/N$, $j=1,\ldots,N$.  Evaluation uses 17 equally spaced centers from 1930 through
2013 for every $N$.  At each center, held-out cross-entropy is averaged with the corresponding
normalized temporal weights; whole-path cross-entropy is the unweighted mean over the 17 centers.
The held-out split provides a fixed empirical approximation to the whole-path temporal risk.  The
reported $k^\star$ is the minimizer of this empirical risk curve, rather than an estimate of
post-selection predictive performance, and no held-out examples are used for gradient updates.

The classifier has four $3\times3$ convolutional blocks with channel widths
$32$--$32$--$64$--$64$.  Each block uses GroupNorm, ReLU, and $2\times2$ max pooling, followed by
global average pooling and a linear two-class head.  The source model is trained for $1{,}500$
updates on the source-time distribution.  Adaptation uses plain mini-batch SGD with learning rate
$0.03$, batch size $256$, no momentum or weight decay, and
\[
  k\in\{0,1,2,4,8,16,32,64,128\}
\]
updates at every visited center.  Larger $k$ extends a fixed taskwise mini-batch prefix within each
$(N,\text{seed})$ pair.  We use five consecutive seeds from $20260805$ through $20260809$.

A log--log fit to the interpolated mean-curve optima gives
$k^\star\propto N^{-1.37}$.  Repeating the fit in $10{,}000$ seed-bootstrap replicates gives a
median exponent of $1.30$ and a 95\% interval $[0.79,1.66]$.  The interval includes the theoretical
value $1$ but does not distinguish it sharply from nearby exponents.  Some $N=4$ replicates select
the largest tested value, so the upper endpoint in \cref{tab:yearbook} is grid-limited.
Across the tested task counts, the estimated optimal total update budgets do not form a stable
plateau.

\begin{table}[H]
  \centering
  \caption{Yearbook empirical evaluation-set optima over five seeds.  Optima are obtained from local
  quadratic interpolation in $\log_2 k$ around the minimum of each mean evaluation curve.  Confidence
  intervals resample complete training seeds and repeat the averaging and interpolation; they do not
  include uncertainty from resampling the fixed evaluation split.}
  \label{tab:yearbook}
  \normalsize
  \begin{tabular*}{\linewidth}{@{\extracolsep{\fill}}rrrr@{}}
    \toprule
    $N$ & grid-selected $k$ & interpolated $k^\star$ & bootstrap 95\% CI \\
    \midrule
    $4$  & $64$ & $79.90$ & $[63.54,128.00]$ \\
    $8$  & $16$ & $13.82$ & $[7.53,38.48]$ \\
    $16$ & $4$  & $4.23$  & $[2.70,13.10]$ \\
    $32$ & $4$  & $5.01$  & $[3.49,14.91]$ \\
    \bottomrule
  \end{tabular*}
\end{table}

We also perform a necessary endpoint check for approximate joint fit.  A model trained on the
equally weighted source and target distributions has mean held-out endpoint cross-entropy $0.084$,
compared with $0.057$ for separately trained endpoint models; using the unrounded seed means, the
mean gap is $0.028$ across five seeds.
This shows low simultaneous endpoint loss, not a shared zero-loss minimizer along the full path.
The experiment further differs from the theorem through finite samples, supervised stochastic
updates, cross-entropy loss, temporal smoothing, and the nonlinear classifier.  It is therefore a
natural-shift consistency check rather than a verification of the assumptions.

\clearpage
\section{Related Work}
\label{app:related-work}

Existing theory studies training time and interference across sequential tasks from different
angles.  Early-stopping theory treats training time as a statistical regularizer
\citep{yao2007early,raskutti2014early,ali2019early}, and stability analyses relate training duration
to generalization \citep{hardt2016faster}.  These results explain why stopping before convergence
can improve prediction from a finite sample.  Our effect is instead present for noiseless population
losses and concerns the state passed from one task to the next.

Domain-adaptation theory studies when transfer across distributions is possible, while practical
methods align features or predictions between source and target domains
\citep{bendavid2010theory,ganin2016dann,hoffman2018cycada,saito2018mcd}.  Gradual self-training
inserts intermediate distributions and studies how their number and placement affect target error
\citep{kumar2020gradual,chen2021indexed,wang2022gradual}.  These analyses represent each local
adaptation by an empirical risk minimization step.  They do not treat the amount of optimization at
each intermediate task as a variable that changes with the resolution of the path.

Curriculum learning studies the order and pace at which examples or tasks are presented
\citep{bengio2009curriculum,kumar2010selfpaced,graves2017curriculum,weinshall2018curriculum,
hacohen2019power}.  Continual-learning methods control interference through parameter
regularization, stored examples, and gradient projections
\citep{kirkpatrick2017ewc,zenke2017si,ritter2018laplace,benzing2022unifying,
lopezpaz2017gem,chaudhry2019agem,farajtabar2020ogd}.  Recent theory studies how training regimes and
regularized updates affect forgetting
\citep{mirzadeh2020training,zhao2024statistical,levinstein2025regularization}, including intermediate
training regimes \citep{lesort2023challenging,graldi2025lazy,mori2025protocols}.  Exact analyses of
linear regression characterize forgetting under fixed and cyclic task sequences
\citep{evron2022catastrophic,swartworth2023cyclic}, while recent work studies continual models under
SGD and task distributions \citep{evron2026sgd,xu2026distribution}.  These works establish that task
order and training regime matter.  We isolate a different question by holding the underlying task
path fixed and coupling its resolution $N$ to the local training time $s_N$.

\end{document}

%% file: figures/budget_curves.tex
\begin{tikzpicture}[x=1bp,y=1bp,inner sep=0pt,outer sep=0pt]
  \node[anchor=south west] at (0,0) {\includegraphics[width=396bp]{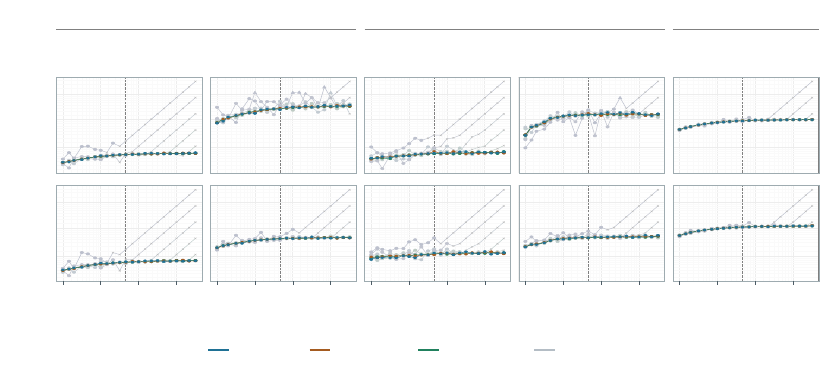}};
  \node[anchor=base,rotate=0,font=\fontsize{8.15}{9.15}\selectfont] at (99.000,165.000) {Rotated MNIST};
  \node[anchor=base,rotate=0,font=\fontsize{8.15}{9.15}\selectfont] at (247.000,165.000) {Rotated Fashion-MNIST};
  \node[anchor=base,rotate=0,font=\fontsize{8.15}{9.15}\selectfont] at (358.000,165.000) {Rotated MNIST};
  \node[anchor=base,rotate=0,font=\fontsize{7.5}{8.5}\selectfont] at (62.000,151.500) {MLP / CE};
  \node[anchor=base,rotate=0,font=\fontsize{7.5}{8.5}\selectfont] at (136.000,151.500) {MLP / squared};
  \node[anchor=base,rotate=0,font=\fontsize{7.5}{8.5}\selectfont] at (210.000,151.500) {MLP / CE};
  \node[anchor=base,rotate=0,font=\fontsize{7.5}{8.5}\selectfont] at (284.000,151.500) {MLP / squared};
  \node[anchor=base,rotate=0,font=\fontsize{7.5}{8.5}\selectfont] at (358.000,151.500) {Frozen trunk / CE};
  \node[anchor=base east,rotate=0,font=\fontsize{7}{8}\selectfont] at (24.000,102.081) {3};
  \node[anchor=base east,rotate=0,font=\fontsize{7}{8}\selectfont] at (24.000,115.587) {10};
  \node[anchor=base east,rotate=0,font=\fontsize{7}{8}\selectfont] at (24.000,127.913) {30};
  \node[anchor=base,rotate=90,font=\fontsize{7.2}{8.2}\selectfont] at (7.000,115.000) {Loss};
  \node[anchor=base,rotate=0,font=\fontsize{7}{8}\selectfont] at (30.181,31.800) {4};
  \node[anchor=base,rotate=0,font=\fontsize{7}{8}\selectfont] at (48.363,31.800) {16};
  \node[anchor=base,rotate=0,font=\fontsize{7}{8}\selectfont] at (66.546,31.800) {64};
  \node[anchor=base,rotate=0,font=\fontsize{7}{8}\selectfont] at (84.727,31.800) {256};
  \node[anchor=base,rotate=0,font=\fontsize{7}{8}\selectfont] at (104.181,31.800) {4};
  \node[anchor=base,rotate=0,font=\fontsize{7}{8}\selectfont] at (122.363,31.800) {16};
  \node[anchor=base,rotate=0,font=\fontsize{7}{8}\selectfont] at (140.546,31.800) {64};
  \node[anchor=base,rotate=0,font=\fontsize{7}{8}\selectfont] at (158.727,31.800) {256};
  \node[anchor=base,rotate=0,font=\fontsize{7}{8}\selectfont] at (178.181,31.800) {4};
  \node[anchor=base,rotate=0,font=\fontsize{7}{8}\selectfont] at (196.363,31.800) {16};
  \node[anchor=base,rotate=0,font=\fontsize{7}{8}\selectfont] at (214.546,31.800) {64};
  \node[anchor=base,rotate=0,font=\fontsize{7}{8}\selectfont] at (232.727,31.800) {256};
  \node[anchor=base,rotate=0,font=\fontsize{7}{8}\selectfont] at (252.181,31.800) {4};
  \node[anchor=base,rotate=0,font=\fontsize{7}{8}\selectfont] at (270.363,31.800) {16};
  \node[anchor=base,rotate=0,font=\fontsize{7}{8}\selectfont] at (288.546,31.800) {64};
  \node[anchor=base,rotate=0,font=\fontsize{7}{8}\selectfont] at (306.727,31.800) {256};
  \node[anchor=base,rotate=0,font=\fontsize{7}{8}\selectfont] at (326.182,31.800) {4};
  \node[anchor=base,rotate=0,font=\fontsize{7}{8}\selectfont] at (344.363,31.800) {16};
  \node[anchor=base,rotate=0,font=\fontsize{7}{8}\selectfont] at (362.546,31.800) {64};
  \node[anchor=base,rotate=0,font=\fontsize{7}{8}\selectfont] at (380.727,31.800) {256};
  \node[anchor=base east,rotate=0,font=\fontsize{7}{8}\selectfont] at (24.000,50.081) {3};
  \node[anchor=base east,rotate=0,font=\fontsize{7}{8}\selectfont] at (24.000,63.587) {10};
  \node[anchor=base east,rotate=0,font=\fontsize{7}{8}\selectfont] at (24.000,75.913) {30};
  \node[anchor=base,rotate=90,font=\fontsize{7.2}{8.2}\selectfont] at (7.000,63.000) {Accuracy};
  \node[anchor=base west,rotate=0,font=\fontsize{7}{8}\selectfont] at (27.000,143.000) {Optimal total time};
  \node[anchor=base,rotate=0,font=\fontsize{7}{8}\selectfont] at (204.000,20.000) {Number of tasks};
  \node[anchor=base west,rotate=0,font=\fontsize{7}{8}\selectfont] at (49.000,5.000) {Learning rate};
  \node[anchor=base west,rotate=0,font=\fontsize{7.1}{8.1}\selectfont] at (113.000,5.000) {0.003125};
  \node[anchor=base west,rotate=0,font=\fontsize{7.1}{8.1}\selectfont] at (161.625,5.000) {0.0015625};
  \node[anchor=base west,rotate=0,font=\fontsize{7.1}{8.1}\selectfont] at (213.800,5.000) {0.00078125};
  \node[anchor=base west,rotate=0,font=\fontsize{7.1}{8.1}\selectfont] at (269.525,5.000) {Five larger rates};
\end{tikzpicture}

%% file: appendix_risk_curvature.tex
\subsection{Common-solution distance and prediction risk}
\label{app:prediction-risk}

We measure learning progress by the reduction in squared Euclidean distance to the common
solution set. Following the convention of removing the common-kernel component, define the
average path risk by
\[
  R(v):=\int_0^1 \cL_t(w^\star+v)\,dt
  =\frac12 v^\top\bar H v,
  \qquad \bar H:=\int_0^1 H(t)\,dt.
\]
By \cref{app:pathwise-risk-equivalence}, $\bar H$ is positive definite on $\cK^\perp$, so
\[
  \frac{\lambda_{\min}(\bar H|_{\cK^\perp})}{2}\|v\|^2
  \le R(v)\le
  \frac{\lambda_{\max}(\bar H|_{\cK^\perp})}{2}\|v\|^2.
\]
These bounds give the distance measure a prediction-loss interpretation, but do not ensure
that different training times have the same ranking under the two measures. Average path
risk also depends on the direction of the error, and training time changes both its length
and direction. \Cref{app:anisotropic-counterexample} illustrates how the evaluation geometry
can affect comparisons between training times. Thus the general conclusion of
\cref{thm:optimal-main} concerns the distance-based progress defined in the main text;
extending it to average path risk requires further consideration of the task-path geometry.

The following sufficient condition covers the rotating-regression example in the main text.

\begin{corollary}[Average-risk scaling under monotone rotation]
\label{cor:monotone-rotation-risk}
Suppose
\[
  H(t)=\lambda(t)n(t)n(t)^\top,
  \qquad n(t)=(\cos\theta(t),\sin\theta(t))^\top,
\]
where $\lambda,\theta\in C^3([0,1])$, $0<\mu\le\lambda(t)\le L$, and $\theta$ is monotone with
\[
  0<|\theta(1)-\theta(0)|\le\frac\pi2.
\]
Let $v_0\in\ker H(0)$ be nonzero. If per-task training time is selected to minimize the final
model's average path risk,
\[
  s_{N,R}^\star\in\operatorname*{arg\,min}_{s\in[0,\infty]}
  R(\Phi_{N,s}v_0),
\]
then, for all sufficiently large $N$, every optimizer is finite and satisfies
\[
  s_{N,R}^\star=\Theta(N^{-1}).
\]
\end{corollary}

\begin{proof}
Rescaling the nonzero initial error does not change the optimal training time, so assume
$v_0=k(0)$, where $k(t)=(-\sin\theta(t),\cos\theta(t))^\top$.
Write $x_\tau(t)=p(t)k(t)+q(t)n(t)$. The moving-frame equations are
\[
  p'=-\theta' q,
  \qquad q'=\theta'p-\tau\lambda q,
  \qquad (p(0),q(0))=(1,0).
\]
The expansion used in \cref{app:tail-details} gives
\[
  p(1)=1-\frac g\tau+O(\tau^{-2}),
  \qquad q(1)=\frac{\theta'(1)}{\tau\lambda(1)}+O(\tau^{-2}),
  \qquad g:=\int_0^1\frac{\theta'(t)^2}{\lambda(t)}\,dt>0.
\]
Let
\[
  a:=k(1)^\top\bar H k(1)>0,
  \qquad b:=k(1)^\top\bar H n(1).
\]
Monotonicity and the bound on the total rotation imply $b\theta'(1)\le0$. Hence, writing
$R_\tau:=R(x_\tau(1))$ and $R_\infty:=R(k(1))$, we obtain
\[
  R_\tau=R_\infty-
  \frac{ag-b\theta'(1)/\lambda(1)}{\tau}+O(\tau^{-2}).
\]
The numerator is strictly positive, so sufficiently large but finite total training times
satisfy $R_\tau<R_\infty$.

At the other endpoint, the small-budget expansion gives
\[
  R_\tau=R_0-\tau v_0^\top\bar H^2v_0+O(\tau^2),
\]
with a strictly positive first-order decrease. Comparing the two endpoints, there is always
a finite $\tau_0>0$ such that
\[
  R_{\tau_0}<\min\{R_0,R_\infty\}.
\]
The endpoint convergence in \cref{thm:continuum-main} and the uniform large-budget estimates
in \cref{app:finite-time-proof} now exclude optimal total times tending to zero or infinity,
by the optimizer-localization argument in the proof of \cref{thm:optimal-time}. This proves
the claimed scaling.
\end{proof}

The constant-speed rotating example in the main text satisfies these conditions. Its
average-path-risk and common-solution-distance objectives therefore yield the same
$N^{-1}$ scaling order, although their optimal proportionality constants need not coincide.
This corollary applies to the linear rotating model. The Rotated MNIST, Rotated Fashion-MNIST, and Yearbook results remain empirical tests beyond
these assumptions. The paired-process guarantee in \cref{thm:paired-orders} instead permits
nonlinear models and other smooth evaluation losses, using a comparison of the two endpoints
and a finite-budget improvement.

\subsection{Small positive curvature and finite training horizons}
\label{app:small-positive-curvature}

The vanishing-progress conclusion at large budgets uses the exactly zero eigenvalues of
the task curvatures. Giving these directions small positive curvature can preserve nearby
trajectories over finite training horizons while changing the infinite-training behavior.
The following exactly solvable perturbation distinguishes these two statements.

Consider
\[
  H_\varepsilon(t)=H(t)+\varepsilon I,
  \qquad \varepsilon>0,
\]
with the same common solution $w^\star$ and initial error $v_0$. Denote the errors after $N$
tasks in the original and perturbed models by $v_N^{(0)}$ and $v_N^{(\varepsilon)}$.
Because $\varepsilon I$ commutes with every $H(t)$,
\[
  e^{-sH_\varepsilon(t)}=e^{-\varepsilon s}e^{-sH(t)},
  \qquad v_N^{(\varepsilon)}=e^{-\varepsilon Ns}v_N^{(0)}.
\]
Writing $\tau=Ns$, contractivity of the original updates gives
\[
  \|v_N^{(\varepsilon)}-v_N^{(0)}\|
  \le(1-e^{-\varepsilon\tau})\|v_0\|
  \le\varepsilon\tau\|v_0\|.
\]
Thus the final states remain close in absolute error when $\varepsilon\tau\ll1$.

This perturbation also changes the source-fitting condition. The original initialization
satisfies $H(0)v_0=0$, whereas its loss in the perturbed model is
\[
  L_0^{(\varepsilon)}(w^\star+v_0)
  =\frac\varepsilon2\|v_0\|^2.
\]
Keeping the same initial model therefore corresponds to approximate source fitting. With
the positive-definite source curvature $H_\varepsilon(0)$, exact source fitting would
already reach the unique common solution.

Closeness of the trajectories and preservation of the original leading-order asymptotics
require different comparisons. Define
\[
  \cP_{N,s}^{(\varepsilon)}
  :=\|v_0\|^2-\|v_N^{(\varepsilon)}\|^2.
\]
The exact identity above yields
\[
  \cP_{N,s}^{(\varepsilon)}-\cP_{N,s}^{(0)}
  =(1-e^{-2\varepsilon\tau})\|v_N^{(0)}\|^2,
\]
and hence
\[
  0\le\cP_{N,s}^{(\varepsilon)}-\cP_{N,s}^{(0)}
  \le2\varepsilon\tau\|v_0\|^2.
\]
In the continuum limit, the same relation gives
$x_\tau^{(\varepsilon)}(1)=e^{-\varepsilon\tau}x_\tau^{(0)}(1)$.
The original large-budget progress satisfies $\cP_\tau^{(0)}\sim C/\tau$ for some $C>0$.
To make the perturbation negligible relative to this leading term, the stronger scale
separation $\varepsilon\tau^2\ll1$ is sufficient. More precisely, in the joint limit
$\varepsilon\to0$, $\tau\to\infty$, and $\varepsilon\tau^2\to0$,
\[
  \cP_\tau^{(\varepsilon)}\sim\frac C\tau.
\]
This identifies a joint asymptotic regime in which the original large-budget mechanism
is retained to leading order.

In contrast, for any fixed $\varepsilon>0$,
\[
  \|v_N^{(\varepsilon)}\|
  \le e^{-\varepsilon\tau}\|v_0\|
  \longrightarrow0
  \qquad(\tau\to\infty).
\]
Sufficiently long training therefore removes the error, and the original conclusion that
large-budget progress vanishes no longer holds. This solvable perturbation shows that
exactly zero and small positive curvatures can produce nearby finite-horizon trajectories
but different infinite-time limits. The global optimal-time theorem relies on the stated
nullspace structure; for models with approximately zero curvature, its applicability must
be assessed jointly with the curvature magnitude and training horizon.

%% file: appendix_paired_budgets.tex
\subsection{A path whose risk-optimal total time diverges}
\label{app:risk-order-example}

\paragraph{Construction.}
Choose a nonnegative smooth bump $\beta$ supported in $(0,1/10)$ with integral one. Set
$g(t)=1-t-a\beta(t)$ with $a=0.6$, $\theta(t)=\varepsilon g(t)$,
$u=(\cos\theta,\sin\theta)$, $n=(-\sin\theta,\cos\theta)$, and $H=nn^\top$.
Start at $v_0=u(0)$ and evaluate $R(v)=v^\top\bar H v/2$, where
$\bar H=\int_0^1H(t)\,dt$. The path is smooth, has constant nonzero eigenvalue one, and has
trivial common kernel. Thus \cref{thm:optimal-main} applies to its distance criterion.
We prove that, for one sufficiently small fixed $\varepsilon>0$,
$R(v_\tau)>R(e_1)$ for every finite $\tau\ge0$, where $v_\tau$ is the continuum endpoint and
$e_1=u(1)$ is the exact-fitting limit.

\paragraph{Uniform control of the endpoint direction.}
Write the continuum error as $p(t)u(t)+q(t)n(t)$. Its equations are
\[
 p'=\varepsilon g'q,\qquad q'=-\tau q-\varepsilon g'p,
 \qquad (p(0),q(0))=(1,0).
\]
Constants below may depend on the fixed bump but not on $\tau$ or sufficiently small
$\varepsilon$. Contractivity gives $|p|\le1$. Variation of constants yields
\[
 q(t)=-\varepsilon\int_0^t e^{-\tau(t-r)}g'(r)p(r)\,dr,
 \qquad |q(t)|\le \frac{C\varepsilon}{1+\tau}.
\]
Integrating $p'=\varepsilon g'q$ then gives
$|p(t)-1|\le C\varepsilon^2/(1+\tau)$. Substituting this bound back at $t=1$ gives
\[
 q(1)=\varepsilon I(\tau)+O\!\left(\frac{\varepsilon^3}{(1+\tau)^2}\right),\qquad
 I(\tau)=\frac{1-e^{-\tau}}{\tau}-a\tau \int_0^1 e^{-\tau(1-t)}\beta(t)\,dt.
\]
The formula for $I$ follows by integration by parts, with the boundary terms of $\beta$
equal to zero. At $\tau=0$, use the continuous extension. Since $\beta$ is supported in
$(0,1/10)$ and has integral one, its integral term is at most $a\tau  e^{-0.9\tau}$.
Also, $(1-e^{-\tau})/\tau=\int_0^1e^{-\tau r}\,dr\ge e^{-\tau/2}$. Hence the ratio of the subtracted
term to the first term is at most
\[
 a\tau  e^{-0.4\tau}\le \frac{a}{0.4e}<1.
\]
This proves $I(\tau)\ge c/(1+\tau)$ whenever $a<0.4e$. With our choice $a=0.6$, choosing
$\varepsilon$ sufficiently small gives $q(1)\ge c'\varepsilon/(1+\tau)>0$ and $p(1)\ge1/2$
uniformly for all finite $\tau$.

\paragraph{Comparing direction and length in the risk.}
Since $\int_0^1g(t)\,dt=1/2-a=-0.1$, Taylor expansion gives
\[
 \bar H_{11}=O(\varepsilon^2),\qquad
 \bar H_{12}=0.1\varepsilon+O(\varepsilon^3),\qquad
 \bar H_{22}=1+O(\varepsilon^2).
\]
At $t=1$, the moving basis is $(e_1,e_2)$, so twice the risk difference is
\[
 2\{R(v_\tau)-R(e_1)\}
 =\bar H_{11}(p(1)^2-1)+2\bar H_{12}p(1)q(1)+\bar H_{22}q(1)^2.
\]
The magnitude of the first term is at most $C\varepsilon^4/(1+\tau)$, the second term is
at least $c\varepsilon^2/(1+\tau)$, and the last is nonnegative. Therefore
\[
 R(v_\tau)-R(e_1)\ge\frac{c\varepsilon^2-C\varepsilon^4}{2(1+\tau)}>0
\]
for a sufficiently small fixed $\varepsilon$, simultaneously for all finite $\tau\ge0$.
The comparison is uniform in sign, although its positive margin tends to zero as $\tau\to\infty$.

\paragraph{From continuum risk to the discrete optimum.}
Exact fitting at the $N$ sampled tasks gives terminal error
\[
 e_1\prod_{j=1}^N\cos\bigl(\theta(t_j)-\theta(t_{j-1})\bigr)\longrightarrow e_1,
\]
because smoothness bounds the sum of squared angle increments by $O(1/N)$.
The optimal discrete risk thus has limsup at most $R(e_1)$. Here we optimize over the
extended interval $s\in[0,\infty]$, including exact fitting by continuity.
If a sequence of optimal total times had a bounded subsequence, a further subsequence would
converge to a finite $\tau$. By \cref{thm:continuum}, its risk would converge to
$R(v_\tau)>R(e_1)$, a contradiction. Thus every sequence of risk-optimal total times diverges,
while \cref{thm:optimal-main} keeps the distance-optimal total times of order one.

\paragraph{Applying \cref{thm:paired-orders} to the reverse order.}
The linear model $f_v(X_t)=n(t)^\top v$ with target zero realizes this path, and its parameter Jacobian $n(t)^\top$ has full row rank everywhere. Sequential exact fitting converges uniformly to $u(t)$ from $u(0)$ and to $u(1-t)$ in reverse from $e_1$: the cumulative squared angle increments are $O(N^{-1})$, so the projection products have length $1+O(N^{-1})$ throughout. The continuous training flows exist for every finite budget because their coefficients are bounded. Thus the nonsingularity condition holds in both orders.

Set $H^+(t)=H(t)$, $H^-(t)=H(1-t)$, $c_+=u(0)$, and $c_-=e_1$. The continuum path $x^d_\tau$ satisfies $\dot x^d_\tau=-\tau H^d(t)x^d_\tau$, $x^d_\tau(0)=c_d$. Both orders have the same integrated Hessian $\bar H$, so, as $\tau\downarrow0$,
\[
x^d_\tau(1)=c_d-\tau\bar H c_d+O(\tau^2),\qquad
F^d(\tau)=R(c_d)-\tau c_d^\top\bar H^2c_d+O(\tau^2).
\]
The trivial common kernel implies $\bar H\succ0$, and both initial points are nonzero. Hence $(F^d)'(0)=-c_d^\top\bar H^2c_d<0$ for both directions, which verifies strict improvement. Taking $\tau=0$ in the risk inequality above gives $R(e_1)<R(u(0))$. The proof of \cref{thm:paired-orders} therefore selects the reverse order, whose optimal per-task time is $\Theta(N^{-1})$ and which admits an asymptotically optimal fixed total budget, even though the forward risk-optimal total time diverges.

\paragraph{Numerical construction and display in \cref{fig:risk-counterexample}.}
The plotted bump is the normalized function
\[
 \beta(t)=C_\beta\begin{cases}
 \exp\!\left[-\dfrac{1}{(10t)(1-10t)}\right],&0<t<1/10,\\
 0,&\text{otherwise},
 \end{cases}
 \qquad \int_0^1\beta(t)\,dt=1.
\]
We use $a=0.6$, $\varepsilon=3\times10^{-4}$, and total times $\tau=5.52,55.2,552$.
To retain the small contraction numerically, write $p=1+\varepsilon^2 h$ and $q=\varepsilon z$.
The integrated equations are $h'=g'z$ and $z'=-\tau z-g'(1+\varepsilon^2h)$, starting at $(0,0)$.
The figure uses an implicit Radau solver with relative tolerance $2\times10^{-10}$,
absolute tolerance $2\times10^{-12}$, and maximum step size 0.002. The 501 background task
lines are sampled uniformly in $t$.

Both the task lines and parameter curves use the same angular display map
\[
 \varphi(\alpha)=c_\varphi\left[0.3\arctan(\alpha/\varepsilon)
                  +0.7\arctan(10\alpha/\varepsilon)\right],
 \qquad
 c_\varphi=\frac{\pi/4}{0.3\arctan(1)+0.7\arctan(10)}.
\]
For a parameter of polar coordinates $(\rho,\alpha)$, the displayed radius is
$1+9500(\rho-1)$. These common magnifications expose angular lag and radial contraction;
the drawing is not a Euclidean-scale plot. The dashed direction is the eigenvector of
$\bar H$ with the smallest eigenvalue, which minimizes prediction loss at fixed length.
Outer arrows are guides to task order, rather than parameter trajectories.

\clearpage
\subsection{Tracking and optimality for the paired task orders}
\label{app:paired-proof}
\paragraph{Coordinates and the exact-fit path.}
For this proof, write $r(t,w)=(f_w(X_t)-y_t)/\sqrt m$ and $J=D_wr$, so that
$L_t(w)=\|r(t,w)\|^2/2$. Write $F_N^d(\tau)=R(w_N^d(\tau/N))$.
Let $\gamma$ be the continuous path of exact fits in assumption (1). We first justify its
smoothness and evolution equation. At a regular exact fit $(t,z)$, let $E_{t+h}(z)$ be the endpoint
of gradient flow on task $t+h$ initialized at $z$. In a sufficiently small regular neighborhood,
the initial residual is $h\partial_t r(t,z)+O(h^2)$, decays exponentially during the frozen flow,
and produces total parameter displacement $O(h)$. Variation of constants then gives, uniformly
on compact regular neighborhoods,
\[
 E_{t+h}(z)=z-hJ(t,z)^\dagger\partial_t r(t,z)+O(h^2).
\]
Indeed, replacing $J$ along the frozen flow by $J(t,z)$ incurs an integrated error $O(h^2)$,
and integrating the linear residual equation yields the displayed Moore--Penrose inverse.
Sequential exact fitting is therefore locally a consistent Euler approximation to
\[
  \dot\gamma(t)=-J(t,\gamma(t))^\dagger\partial_t r(t,\gamma(t)),\qquad \gamma(0)=a.
\]
Local uniqueness identifies this solution with the pointwise exact-fit limit. If its maximal
continuation stopped before reaching $t=1$, continuity of the assumed limit path would give a
finite endpoint where $J$ still has full row rank, so the local solution could be extended.
Thus $\gamma$ is smooth on the whole interval, with $\gamma(1)=b$. Its velocity is the
minimum-norm velocity preserving $r(t,\gamma(t))=0$.

Use $h=1/N$, $s=\tau/N$, and treat the forward order first. Continuity of $\gamma$ on $[0,1]$
makes its graph compact; full row rank and continuity of $J$ give a uniform positive lower bound
on its smallest singular value there. These facts give a uniform local tube in which the
nearest-point map $\pi_t$ onto $M_t$ is
smooth with bounded first and second derivatives (by the implicit function theorem in normal
coordinates).  On the manifold, $D_w\pi_t$ is the tangent projector.  Write
$V(t,z)=-J(t,z)^\dagger\partial_t r(t,z)$.  At $z\in M_t$,
$\partial_t\pi_t(z)=V(t,z)$: differentiating the constraint gives the required normal component,
and differentiating the normality of $z-\pi_t(z)$ at zero displacement gives zero tangent
component.  In the tube, distance to $M_t$ is comparable to $\|r(t,w)\|$.
All bounds below are local to this tube, and the tracking estimate ensures the trajectory stays
there for the large-budget sequences considered.

\paragraph{Uniform residual control.}
During a frozen task the residual satisfies $\dot r=-JJ^\top r$, so
$\|r(u)\|\le e^{-\mu u}\|r(0)\|$ for a uniform $\mu>0$.  A task switch changes the residual
norm by at most $Ch$.  Hence its maximum, including the starts of tasks, is at most
$D=Ch/(1-e^{-\mu s})$ by the geometric series.  Along sequences $N\to\infty$, $Ns\to\infty$,
$D\to0$.  In addition, summing the drops of the residual norm over all frozen stages gives at
most $NCh$ since the initial residual is zero.  The inequality
$-d\|r\|/du\ge\mu\|r\|$ on each stage therefore gives
$\sum_j\int_{\mathrm{stage}\ j}\|r\|\,du\le C/\mu$.

\paragraph{Tracking and endpoint exchange.}
Put $z_j=\pi_{t_j}(w_j)$.  At the switch from $t$ to $t+h$, Taylor expansion of the projection
changes it by $hV(t,\pi_t(w))+O(hD+h^2)$.  During a frozen stage the projected velocity has norm
at most $C\|r\|^2$: $D_w\pi_t(w)J(t,w)^\top$ vanishes on the manifold and is Lipschitz in the
tube.  Its integrated displacement over all stages is therefore $O(D)$.  The projected points
follow the Euler recursion for $\dot\gamma=V(t,\gamma)$ with total absolute local errors
$O(D+h)$.  The discrete Gronwall inequality gives $\max_j\|z_j-\gamma(t_j)\|=O(D+h)$.
The same bounds apply up to any first exit, including a partial frozen stage; for sufficiently
large $N$ the projected and normal errors together are smaller than the tube radius, precluding
such an exit.  The terminal state tends to $b$.  Reversing the path reverses $V$ and yields the
tracked curve $\gamma(1-t)$ from $b$ to $a$.

\paragraph{Finite budgets and localization of the optimum.}
For bounded $\tau$, freezing the gradient field at each right endpoint has local error $O(h^2)$ and
global error $O(h)$, uniformly on compact budget intervals by smooth dependence and Gronwall's
inequality.  Thus $F_N^d\to F^d$ uniformly there, including at $\tau=0$.
Choose the order starting at the endpoint with lower $R$, and a finite budget beating that starting
value.  Its limiting loss lies strictly below both boundary losses.  The sublevel sets at its
discrete loss lie in a common compact subinterval of $(0,\infty)$ for large $N$: otherwise a
sequence in these sublevel sets approaches zero or escapes to infinity, contradicting the
corresponding endpoint limit.  Continuity gives finite minimizers there.
Any accumulation point $\tau_R$ of their total budgets is positive and finite.  Passing to the
limit in the optimality comparison with each fixed budget shows $F^d(\tau_R)=\inf F^d$.
Every subsequence of discrete optimal values has a further subsequence with this same limit.
It follows that $F_N^d(\tau_R)-\min F_N^d\to0$, proving \cref{thm:paired-orders}.